\documentclass[twoside]{article}

\usepackage[accepted]{aistats2025}

\usepackage[utf8]{inputenc} 
\usepackage[T1]{fontenc}    
\usepackage{hyperref}       
\usepackage{url}            
\usepackage{booktabs}       
\usepackage{amsfonts}       
\usepackage{nicefrac}       
\usepackage{xcolor}         
\usepackage{graphicx}
\usepackage{algorithm}
\usepackage{algorithmic}

\usepackage{amsmath}
\usepackage{amssymb}
\usepackage{mathtools}
\usepackage{amsthm}
\usepackage{tikz}
\usepackage{enumitem}

\definecolor{olive}{rgb}{0.6, 0.6, 0.2}
\definecolor{sand}{rgb}{0.8666666666666667, 0.8, 0.4666666666666667}
\definecolor{wine}{rgb}{0.5333333333333333, 0.13333333333333333, 0.3333333333333333}
\definecolor{deblue}{RGB}{11,132,147}
\definecolor{ocra}{RGB}{204, 119, 34}

\usepackage[noabbrev,capitalise,nameinlink]{cleveref}
\hypersetup{
    colorlinks=true,
    linkcolor=[RGB]{0, 0, 139},    
    citecolor=[RGB]{0, 100, 0}     
}

\usepackage{wrapfig}
\usepackage{multirow}

\usepackage{setspace}
\usepackage{subcaption}
\usepackage{lipsum} 

\usepackage[round]{natbib}

\theoremstyle{plain}
\newtheorem{theorem}{Theorem}[section]
\newtheorem{proposition}[theorem]{Proposition}

\theoremstyle{definition}

\theoremstyle{remark}

\usepackage[textsize=tiny]{todonotes}

\usepackage{amsmath,amsfonts,bm}

\def\eqref#1{equation~\ref{#1}}

\def\1{\bm{1}}

\DeclareMathAlphabet{\mathsfit}{\encodingdefault}{\sfdefault}{m}{sl}
\SetMathAlphabet{\mathsfit}{bold}{\encodingdefault}{\sfdefault}{bx}{n}

\begin{document}

\twocolumn[
\runningtitle{Generative Flow Ant Colony Sampler (GFACS)}
\runningauthor{Minsu Kim$^{*}$, Sanghyeok Choi$^{*}$, Hyeonah Kim, Jiwoo Son, Jinkyoo Park, Yoshua Bengio}
\aistatstitle{Ant Colony Sampling with GFlowNets for \\Combinatorial Optimization}

\aistatsauthor{ Minsu Kim$^{*}$ \And Sanghyeok Choi$^{*}$ \And Hyeonah Kim$^{\dagger}$ }

\aistatsaddress{ Mila, Universit\'e de Montr\'eal, KAIST \And KAIST \And Mila, Universit\'e de Montr\'eal   }

\aistatsauthor{Jiwoo Son \And Jinkyoo Park \And Yoshua Bengio}

\aistatsaddress{Omelet \And KAIST, Omelet \And Mila, Universit\'e de Montr\'eal  }

]

\begin{abstract}
We present the Generative Flow Ant Colony Sampler (GFACS), a novel meta-heuristic method that hierarchically combines amortized inference and parallel stochastic search. Our method first leverages Generative Flow Networks (GFlowNets) to amortize a \emph{multi-modal} prior distribution over combinatorial solution space that encompasses both high-reward and diversified solutions. This prior is iteratively updated via parallel stochastic search in the spirit of Ant Colony Optimization (ACO), leading to the posterior distribution that generates near-optimal solutions. Extensive experiments across seven combinatorial optimization problems demonstrate GFACS's promising performances.
\end{abstract}

\section{Introduction}
Combinatorial optimization (CO) aims to optimize an objective function over discrete, finite combinatorial structures under given constraints. It has broad applications in operations research~\citep{zhang2019review}, scientific discovery~\citep{naseri2020application} and machine learning~\citep{pmlr-vR0-chickering95a, chickering2004large}. Traditionally, integer programming approaches have been extensively studied to solve CO problems~\citep{miller1960integer}. While these methods offer guaranteed optimality, they often become impractical for large-scale instances due to their exponential complexity. An alternative is to design problem-specific heuristics, but it requires domain knowledge and labor-intensive design, limiting flexibility to variant problems.

\begin{table}[b!]
\vspace*{-3.75ex}
\footnotesize
\urlstyle{same}
\rule{0.8in}{0.4pt}\\[0.75ex]
* Equal contribution. $\dagger$ Work done while Hyeonah Kim was at KAIST. Code available at: \url{https://github.com/ai4co/gfacs}. Correspondence to: Minsu Kim (\textit{minsu.kim@mila.quebec}) and Sanghyeok Choi (\textit{sanghyeok.choi@kaist.ac.kr}).
\end{table}

Recent advancements in deep learning~\citep{lecun2015deep} brought alternative directions to CO. For a more detailed review, refer to our related works in~\cref{app:extended} and the survey by~\citet{bengio2021machine}. Among deep learning for CO methods, supervised learning (SL) approaches distill mathematical programming methods or heuristics into powerful deep neural networks, enabling fast and scalable solution generation~\citep{vinyals2015pointer, joshi2021learning, sun2023difusco, luo2023neural}. However, SL methods rely on expert-labeled data, limiting their applicability as general-purpose CO solvers.

Reinforcement learning (RL) is another popular approach to CO~\citep{khalil2017learning, kool2018attention, kwon2020pomo,ye2023deepaco}. 
Unlike supervised learning, RL methods train deep neural networks directly using the objective function of a CO problem as a reward, without requiring expert-labeled data. This makes RL methods particularly useful because they can adapt flexibly to a wide range of CO problems, even where specific heuristics are unavailable or unexplored. However, RL requires extensive exploration of vast solution space and often involves non-trivial credit assignments across multiple decisions, which introduces significant complexity to the learning process compared to SL~\citep{minsky1961steps}. 

Hierarchical methods that leverage RL in a two-stage problem-solving procedure tackle such exploration and credit assignment problems through hierarchical decomposition. Previous works by \citet{bello2016neural}, \citet{hottung2021efficient}, and \citet{son2023meta} pretrain large models on massive problem instances using RL and fine-tune these models on specific problem instances at test time using a method called active search (AS). \citet{ye2023deepaco} pretrain relatively small models with significantly reduced training times. 
They pretrain a neural network with RL to produce a prior distribution over solution space and iteratively update the distribution using ant colony optimization (ACO), an effective parallel stochastic search algorithm.

\textbf{Contribution.} In this work, we propose a novel probabilistic meta-heuristic called Generative Flow Ant Colony Sampler (GFACS), which leverages the synergy between a multi-modal amortized prior and iterative update of it using parallel stochastic search. Unlike previous approaches that use RL pretraining aimed at maximizing expected rewards, we employ Generative Flow Networks~\citep[GFlowNets;][]{bengio2021flow,bengio2023gflownet}, designed to learn a distribution where the probability of a sample is proportional to its reward, \textit{i.e.}, $p(x) \propto R(x)$. 
This leads to a multi-modal prior distribution of solutions, allowing it to benefit more from subsequent iterative posterior updates compared to the unimodal priors pretrained with reward maximizing RL, as illustrated in~\cref{fig:multi-modal}.
Moreover, to enhance and stabilize GFlowNet training, we introduce several new techniques that have been verified to be significantly more effective than existing GFlowNet training methods.

Our posterior update rules follow the most progressive methods of ant colony optimization (ACO)~\citep{ye2023deepaco}, creating synergies between our multi-modal amortized prior and their iterative refinement via parallel stochastic search. Additionally, we also validate our method with another posterior search method, active search~\citep[AS;][]{hottung2021efficient}, demonstrating the flexibility of our approach.

\begin{figure}
    \centering
    \includegraphics[width=0.85\linewidth]{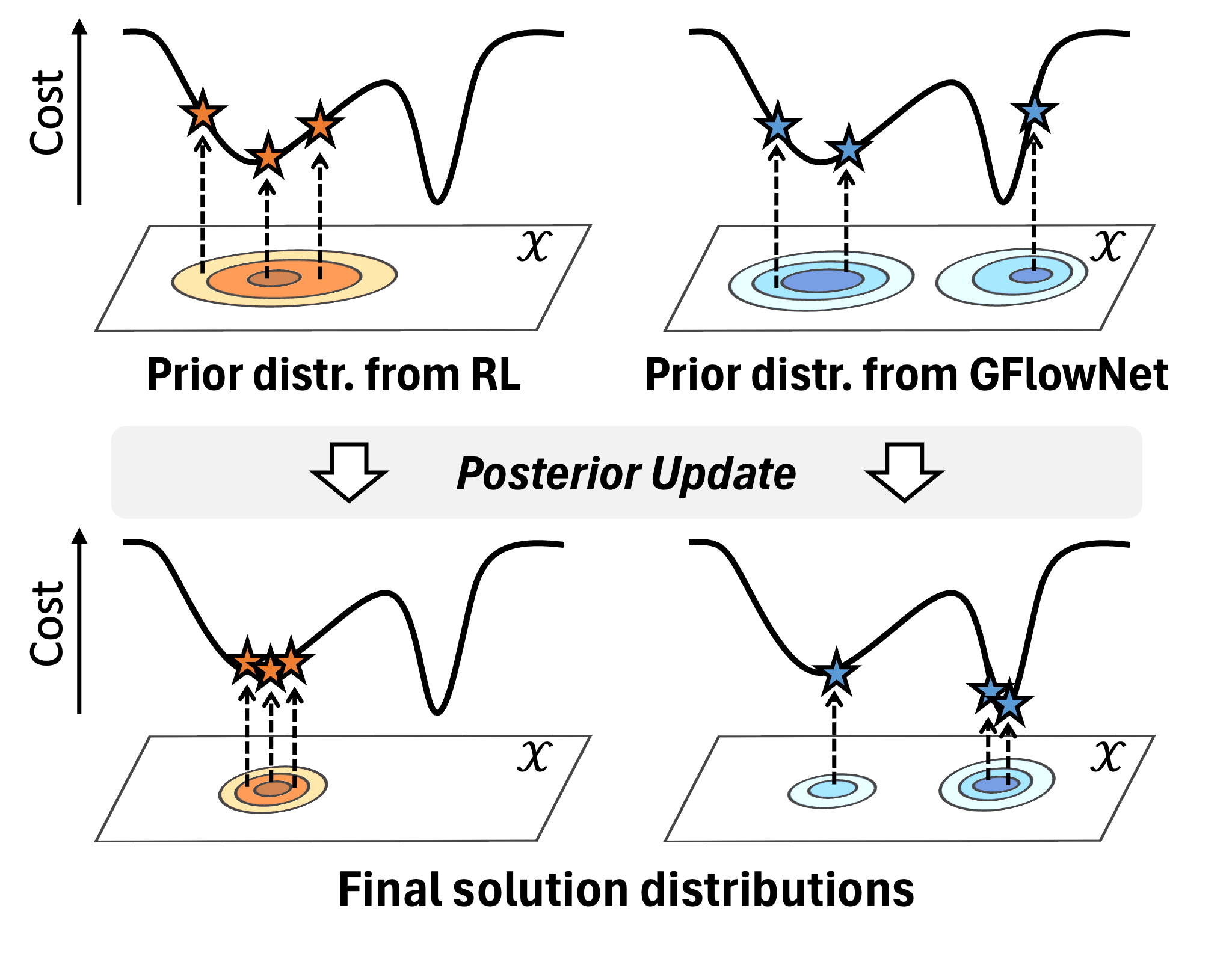}
    \vspace{-1em}
    \caption{Synergy of multi-modal prior distribution of solutions and iterative posterior update with parallel stochastic search in combinatorial optimization.}
    \label{fig:multi-modal}
    \vspace{-1.2em}
\end{figure}

Through extensive experiments, we demonstrate the superiority of our training methods over general-purpose reinforcement learning algorithms for providing effective prior distribution on ACO search across seven combinatorial optimization benchmarks:
\begin{itemize}[left=0pt,nosep]
    \item \textbf{Routing problems}:
    \begin{itemize}[left=0pt,nosep]
        \item Traveling Salesman Problem (TSP)
        \item Capacitated Vehicle Routing Problem (CVRP)
        \item CVRP with Time Windows (CVRPTW)
        \item Prize-Collecting TSP (PCTSP)
        \item Orienteering Problem (OP)
    \end{itemize}
    \item \textbf{Scheduling problem}: Single Machine Total Weighted Tardiness Problem (SMTWTP)
    \item \textbf{Grouping problem}: Bin Packing Problem (BPP)
\end{itemize}

For problems such as the Traveling Salesman Problem (TSP) and the Capacitated Vehicle Routing Problem (CVRP), we compare our GFACS method against problem-specific reinforcement learning algorithms that utilize significantly larger and specialized deep networks.

\section{Preliminary}
\vspace{-0.3em}
\subsection{Traveling salesman problem}

The traveling salesman problem (TSP) is a fundamental combinatorial optimization problem, which serves as a building block for more complex CO problems like the capacitated vehicle routing problem and orienteering problem.
In light of this and for simplicity, this paper presents formulations using the two-dimensional Euclidean TSP as an illustrative example. 

A TSP instance is defined on a (fully connected) graph $\mathcal{G} = (\mathcal{V},\mathcal{D})$, where $\mathcal{V}$ and $\mathcal{D}$ are a set of nodes and edges, respectively. Each node $i \in \mathcal{V}$ has two-dimensional coordinates $v_i \in [0, 1]^{2}$ as a feature, and each edge $(i, j) \in \mathcal{D}$ has a feature $d_{i,j}$ which is the distance between node $i$ and $j$, i.e., $d_{i,j} = \Vert v_i - v_j \Vert_2$.
A solution to TSP, $x$ corresponds to a Hamiltonian cycle for a given problem instance $\mathcal{G}$, which can be represented as a permutation of nodes $\pi = (\pi_1, \ldots, \pi_N)$. The tour length is computed as $\mathcal{E}(x;\mathcal{G}) = d_{\pi_N,\pi_1} + \sum_{t=1}^{N-1}d_{\pi_t,\pi_{t+1}}$. The optimization objective for TSP is to find $x$ that minimizes $\mathcal{E}(x; \mathcal{G})$.

\begin{figure*}[t!]
    \centering
    \includegraphics[width=0.8\linewidth]{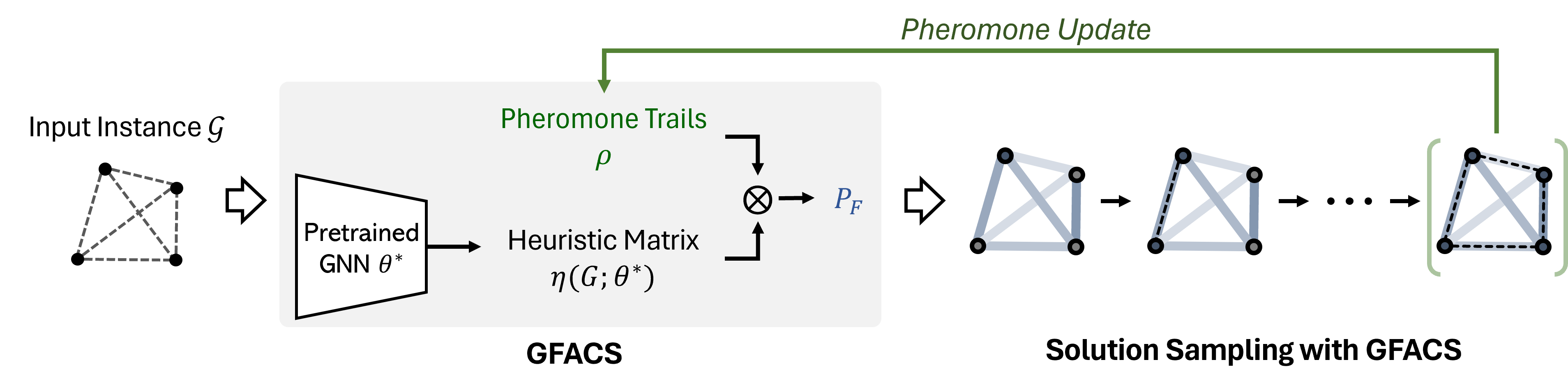}
    \caption{Solution sampling process of GFACS. The GNN, pretraind with a GFlowNet loss, serves as an expert heuristic to pick the next action to construct a solution, such as a tour in the TSP.}
    \vspace{-0.7em}
    \label{fig:GFACS_pipeline}
\end{figure*}

\subsection{Generative flow networks}
\label{sec:gfn}
Generative Flow Networks, or GFlowNets \citep{bengio2021flow,bengio2023gflownet}, provide an effective training framework for learning a policy that samples compositional objects $x \in \mathcal{X}$. The objective is to generate $x$ with probability proportional to a given non-negative reward function $R(x)$. GFlowNets follow a constructive generative process, using discrete \textit{actions} to iteratively modify a \textit{state} which represents a partially constructed object. All the ways to construct these objects can be described by a directed acyclic graph (DAG), \(\mathcal{H}=(\mathcal{S},\mathcal{A})\), where \(\mathcal{S}\) is a finite set of all possible states, and \(\mathcal{A}\) is a subset of \(\mathcal{S}\times\mathcal{S}\), representing directed edges. Within this framework, we define the \textit{children} of state \(s\in\mathcal{S}\) as the set of states connected by edges whose head is \(s\), and the \textit{parents} of state \(s\) as the set of states connected by edges whose tail is \(s\).

We define a (complete) \textit{trajectory} \(\tau = (s_0 \rightarrow \ldots \rightarrow s_N ) \in \mathcal{T}\) from the initial state \(s_0\) to terminal state \(s_n\). We define the \textit{forward policy} to model the forward transition probability \(P_F(s'|s)\) from \(s\) to its child \(s'\). Similarly, we also consider the \textit{backward policy} \(P_B(s|s')\) for the backward transition \(s' \dashrightarrow s\), where \(s\) is a parent of \(s'\). We can use the \textit{forward policy} to compute in a forward way the probability of a trajectory: $P_F(\tau;\theta) = \prod_{t=1}^n P_F\left(s_t|s_{t-1};\theta\right)$. Similarly, we can use the \textit{backward policy} to compute in a backward way the probability of trajectory ending at a given object $x=s_n$: $P_B(\tau|x;\theta) = \prod_{t=1}^n P_B\left(s_{t-1}|s_t;\theta\right)$.

\paragraph{Trajectory balance \citep{malkin2022trajectory}.} Trajectory balance (TB) is the most widely used training objective for GFlowNets. It works with three key components: a learnable scalar $Z_{\theta}$ for the partition function, a learnable forward policy \(P_F(\tau;\theta)\), and a backward policy \(P_B(\tau|x)\) that can either be learnable or fixed with a predefined configuration~\citep{zhang2022generative}. TB loss to minimize for any trajectory $\tau$ is defined as:
\begin{equation}
\mathcal{L}_{\text{TB}}(\tau;\theta) = \left(\log \frac{Z_\theta P_F(\tau;\theta)}{R(x)P_B(\tau|x)}\right)^2.
\label{eq:tb_loss}\end{equation}
When the trajectory balance loss is minimized to zero for all complete trajectories, the probability of sampling $x$ with the forward policy is \(P_{F}^{\top}(x) \propto R(x)\), as desired, where $P_{F}^{\top}(x) = \sum_{(\tau \rightarrow x) \in \mathcal{T}} P_F(\tau)$ with $(\tau \rightarrow x)$ indicating that the terminal state of $\tau$ is $x$.

\subsection{Ant colony optimization} \label{sec:aco}

Ant colony optimization~\citep[ACO;][]{dorigo2006aco} is a meta-heuristic to refine the distribution over solution space by an iterative procedure of solution construction and pheromone update. Given a problem instance $\mathcal{G}$, a corresponding heuristic matrix $\eta(\mathcal{G}) \in \mathbb{R}^{N \times N}$ and a pheromone trails $\rho \in \mathbb{R}^{N \times N}$, a solution is constructed by auto-regressively sampling edges from the probability distribution:
\begin{equation} \label{eq:aco_prob}
   P(\pi_{t+1} = j|\pi_{t} = i; \rho,\eta(\mathcal{G})) \propto \begin{cases} \rho_{i,j}\eta_{i,j}(\mathcal{G}) & \text{if $j \notin \pi_{1:t}$ }\\ 0 & \text{otherwise},
   \end{cases}  
\end{equation}
where $\rho_{i,j}$ and $\eta_{i,j}(\mathcal{G})$ are the pheromone weight and heuristic weight for edge $d_{i,j}$, respectively.

At each iteration, $K$ artificial ants independently sample solutions $\{\pi^k\}_{k=1}^K$ by following Eq.(\ref{eq:aco_prob}). The pheromone trails $\rho$ are then adjusted using the sampled solutions and their corresponding costs, thereby updating the solution distribution. The exact pheromone update rules depend on the specific ACO algorithm being used.
As a representative example, in Ant System \citep{ant-sys}, all $K$ solutions contribute to the pheromone update as follows:
\begin{equation} \label{eq:ant_system}
    \rho_{i,j} \leftarrow (1 - \gamma)\rho_{i,j} + \sum\limits_{k=1}^{K} \Delta^k_{i,j},
\end{equation}
where $\gamma$ is the decaying coefficient (or evaporation rate) and $\Delta^k_{i,j}$ is the amount of pheromone laid by the $k$-th ant on the edge $(i, j)$. The value of $\Delta^k_{i,j}$ is proportional to the inverse of tour length if edge $(i, j)$ is part of $\pi^k$, and zero otherwise:
\begin{equation} \label{eq:as_pheromone}
       \Delta^k_{i,j} = \begin{cases} C / f(\pi^k) & \text{if $(i, j) \in \pi^k$} \\ 0 & \text{otherwise}, \end{cases}
\end{equation}
where $C$ is a constant. The smaller the cost $f(\pi^k)$, the more pheromone is laid on the edges of $\pi^k$. This mechanism encourages future ants to follow promising paths that lead to high-quality solutions.

Note that the heuristic matrix $\eta(\mathcal{G})$ is not updated during the search process; it only serves as an initial \textit{prior}. This prior is augmented with the pheromone $\rho$ that is typically initialized as a constant matrix and then iteratively updated, resulting in the posterior solution distribution that converges to the Dirac-delta distribution of optimal solution under certain conditions~\citep{dorigo1997ant, dorigo2005ant}.

In DeepACO \citep{ye2023deepaco}, the learned graph neural network (GNN), parameterized by $\theta$, is utilized to construct the heuristic matrix $\eta(\mathcal{G},\theta)$ for a given problem instance $\mathcal{G}$, providing an informative prior for guiding probabilistic solution search process of ACO. The limitation of the approach is that they use REINFORCE~\citep{williams1992simple} for GNN training, which suffers from mode-collapse, restricting further improvement from the subsequent iterative pheromone update. Our work seeks to overcome this limitation by incorporating a diversified, multi-modal solution distribution, which shares similar motivation with approximate global optimization methods such as simulated annealing \citep{kirkpatrick1983optimization}.

\section{Generative flow ant colony sampler}

In this section, we provide the inference procedure for the Generative Flow Ant Colony Sampler (GFACS). GFACS enhances Ant Colony Optimization (ACO) by integrating it with multi-modal prior from GFlowNet-based training.

\subsection{MDP formulation} \label{sec:mdp}

We illustrate the solution construction procedure as a Markov decision process (MDP), using the Traveling Salesman Problem (TSP) as an example. This formulation can be easily adapted to general combinatorial optimization problems with minor modifications.

\textbf{State $s$:} The initial state $s_0$ is an empty set. The state $s_t$ at $0<t<n$ is a partial combinatorial solution presented with previous actions: $\{a_0, \ldots, a_{t-1}\}$ and $s_n$ is a completed solution.

\textbf{Action $a$:} The action $a_t$ is a selection of the next node to visit at time $t$ among un-visited nodes: $a_t \in \{1,\ldots,N\}\setminus \{a_0,\ldots,a_{t-1}\}$. If all nodes are visited, a deterministic terminating action is applied, which maps a final state $s_N$ into a solution, $x$.

\textbf{Reward $R$:} The reward is given by $R(x;\mathcal{G},\beta) = e^{-\beta \mathcal{E}(x;\mathcal{G})}$. Here, $\mathcal{E}(x;\mathcal{G})$ represents the energy of the solution $x$, which corresponds to metrics such as the tour length in the TSP. The inverse temperature $\beta$ is a hyperparameter. For simplicity, we will omit $\mathcal{G}$ from $R$ and $\mathcal{E}$ except when its inclusion is necessary for clarity.

\subsection{Policy configuration for prior modeling} \label{sec:policy}

We parameterize the forward policy $P_F$ using a graph neural network (GNN) with parameter $\theta$ that maps an input graph instance $\mathcal{G}$ to the output edge matrix $\eta(\mathcal{G};\theta)$, following \citet{ye2023deepaco}. For simplicity, we use notation $\eta(\mathcal{G})$ without $\theta$. With a given $\eta(\mathcal{G})$, the forward policy when $t < N$ is expressed as follows:
\begin{equation} \label{eq:policy}
    P_F(s_{t+1}|s_t; \eta(\mathcal{G}), \rho)
    \propto \begin{cases} \rho_{a_{t-1},a_{t}}\eta_{a_{t-1},a_{t}}(\mathcal{G}) & \text{if $a_t \notin s_{t}$ }\\ 0 & \text{otherwise} \end{cases}
\end{equation}
Note that $\rho$ and $\eta(\mathcal{G})$ denote the pheromone trails and learned heuristic matrix in ACO. Subsequently, we can define a trajectory from the initial state to a terminal solution, with $\tau = (s_0, \ldots, s_N, x)$, and the forward probability of that trajectory as
\begin{equation} \label{eq:forward_traj}
    P_F(\tau;\eta(\mathcal{G}), \rho) = \prod_{t=0}^{N-1} P_F(s_{t+1}|s_t; \eta(\mathcal{G}), \rho).
\end{equation}
The forward probability for the transition from $s_N$ to $x$ is omitted since it is deterministic. During training, we fix $\rho$ to $\mathbf{1}$, and thus we omit it from subsequent sections unless its presence is required for clarity.

For the backward policy, we first define $h(x)$ as the number of possible trajectories that yield the solution $x$, and we set $P_B(\tau|x) = \frac{1}{h(x)}$. This setting emulates the maximum entropy GFlowNets~\citep{zhang2022generative}, treating all symmetric trajectories that terminate in the same solution equally. In the case of the TSP, where $h(x)=2N$, we have $P_B=\frac{1}{2N}$, thereby imposing a uniform backward distribution across all $2N$ feasible trajectories leading to the same Hamiltonian cycle $x$. It is important to note that the number of symmetric trajectories varies depending on the specific CO problem; for a more detailed discussion on solution symmetry across various CO problems, please refer to \Cref{app:symmetric}

\subsection{ACO update on learned prior}

In the sampling process, we assume the existence of the learned GNN parameterized with $\theta^*$; see \cref{sec:training} for details of the training procedure. 
This learned GNN provides the prior parameter $\eta^{*}(\mathcal{G})$ for the ACO algorithm. The next step is to update the pheromone parameter $\rho$, using $K$ agents (ants) to search over samples drawn from the amortized prior $P_F(\tau; \eta^{*}(\mathcal{G}))$ as described in~\cref{sec:aco}. 

Since the trained $\eta^{*}(\mathcal{G})$ samples diverse, high-reward (\textit{i.e.}, low-energy) candidates, the distribution $P_F$ provides strong prior candidates for the ants in the ACO algorithm. This learned prior is expected to highly improve the likelihood of discovering the global optimum after the iterations of the pheromone update.

The overall process is illustrated in \Cref{fig:GFACS_pipeline}.

\section{Off-policy training of GFACS} \label{sec:training}
This section introduces how to train GFACS, especially the GNN component parameterized with $\theta$, in the GFlowNet framework. 
GFACS is trained via iterations of the off-policy experience collection (\cref{sec:stepA}) and training (\cref{sec:stepB}).

\subsection{Off-policy experience collection} \label{sec:stepA}

\begin{figure}[t]
    \centering
    \includegraphics[width=0.9\linewidth]{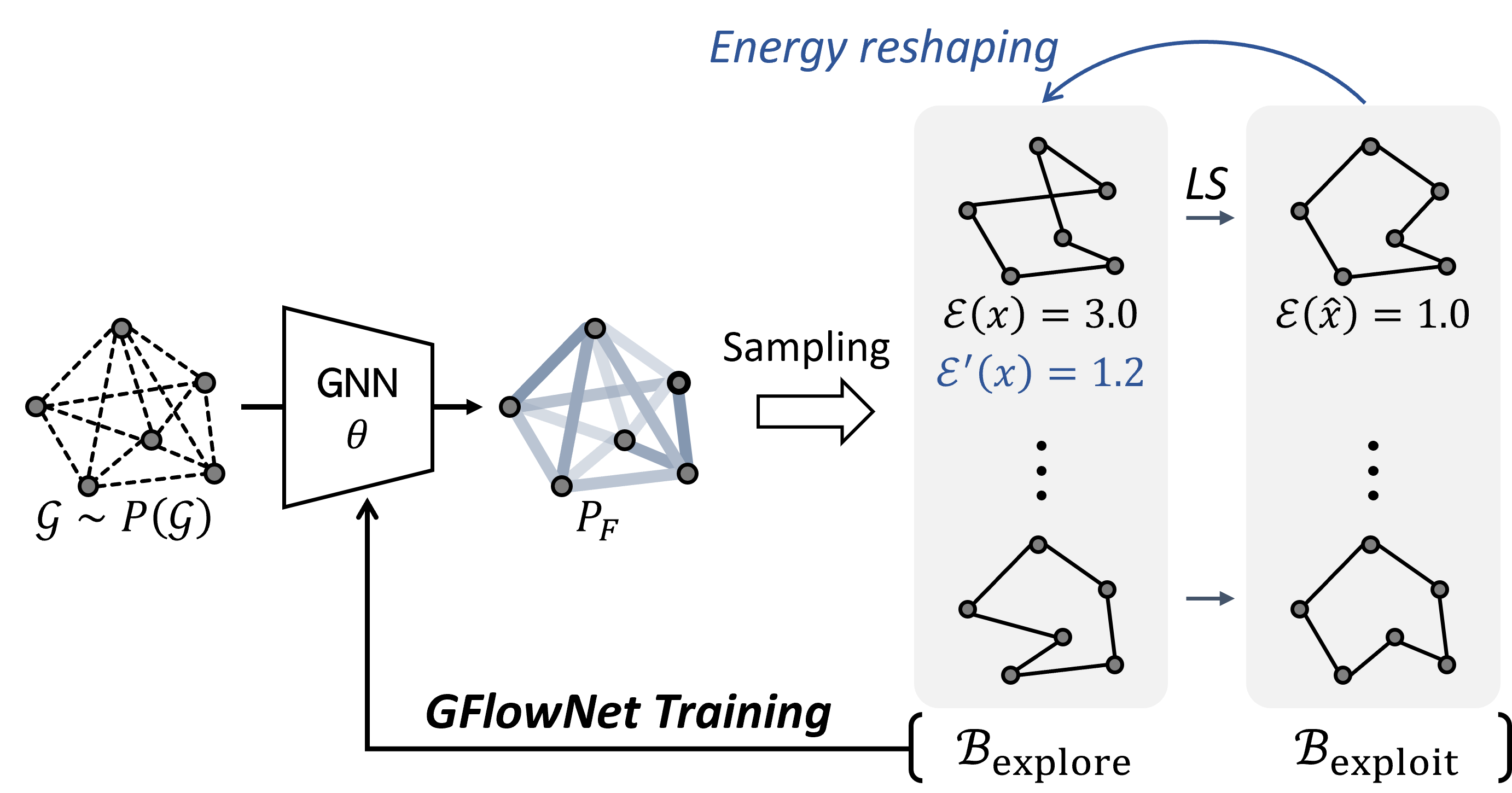}
    \caption{Experience collection procedure. Energy reshaping compensates for the energy of underrated samples that have potential to become low-energy samples after the local search.}
    \label{fig:GFACS_training}
\end{figure}

In this step, we collect experiences on given graph instances. The collected experiences will be used to train GFACS, as described in~\cref{sec:stepB}. At each training round, we first sample $M$ problem instances from a random instance generator $P(\mathcal{G})$, i.e., $\mathcal{G}_1, \ldots, \mathcal{G}_{M} \sim P(\mathcal{G})$. Details of instance generation are described in \cref{app:problems}. For each problem instance $\mathcal{G}_m$, we collect $K$ solutions by following Eq.(\ref{eq:forward_traj}), resulting in a batch of ``exploration experiences,''
\begin{equation*}
\mathcal{B}_{\text{explore}} = \left\{ \left( \mathcal{G}_m, \tau_{m, k}, \mathcal{E}(x_{m,k}) \right) \mid m \in [M], k \in [K] \right\}.
\end{equation*}

To exploit high-reward regions in the large-scale combinatorial space, we collect additional $K$ experiences for each $\mathcal{G}_m$ using a local search (LS) operator, such as 2-opt for TSP, as a powerful behavior policy. 
The LS operator directly refines a given solution into a locally optimized one, \textit{i.e.}, $\hat{x} \leftarrow \text{LS}(x)$. 
We then use the backward policy $P_B$ to sample a trajectory $\hat{\tau}$ for each $\hat{x}$, 
obtaining a batch of ``exploitation experiences,'' 
\begin{equation*}
\mathcal{B}_{\text{exploit}} = \left\{ \left( \mathcal{G}_m, \hat{\tau}_{m, k}, \mathcal{E}(\hat{x}_{m,k}) \right) \mid m \in [M], k \in [K] \right\}.
\end{equation*}

For problems that lack domain-specific LS operators, we also suggest a general-purpose destroy-and-repair method that partially destroys a solution and then repairs; see \cref{append:destroy-and-repair} for more details.

While the concept of using local search for exploitation has been explored in recent GFlowNet works~\citep{kim2023local}, our distinct approach involves utilizing the local search operation to facilitate \textit{hindsight observation}—a forward-looking insight derived from extending the search beyond the current state. Methodologically, we compensate for the solutions that were initially overlooked due to high energy (\textit{i.e.}, low reward) but that can be readily transformed into low energy (\textit{i.e.}, high reward) solutions through local search.

To achieve this, we modify the energy (in a way we call \textit{energy reshaping}) of $x$ in the \textit{exploration} experiences, $\mathcal{B}_{\text{explore}}$. Specifically, for a given solution $x\in\mathcal{B}_{\text{explore}}$ and its corresponding refined solution $\hat{x} \gets \text{LS}(x)$, we adjust the energy of $x$ as follows:
\begin{equation} \label{eq:energy_reshaping}
\mathcal{E}'(x) = \alpha \mathcal{E}(\hat{x}) + (1-\alpha) \mathcal{E}(x).
\end{equation}
This reshaping captures the potential for samples to become low-energy solutions after local search.

The overall procedure of experience collection via sampling and local search, along with the energy reshaping, is illustrated in \cref{fig:GFACS_training}.

\subsection{Training with collected experiences} \label{sec:stepB}

Given a batch of experiences $\mathcal{B}$, the trajectory balance (TB) loss in Eq.(\ref{eq:tb_loss}) can be rewritten as
\begin{equation} \label{eq:tb_ours}
   L(\mathcal{B}) = \frac{1}{MK}\sum_{m=1}^{M} \sum_{k=1}^{K} \left( \log \frac{ P_F(\tau_{m, k};\eta(\mathcal{G}_m))Z(\mathcal{G}_m)}{(1/h(x_{m,k}))R(x_{m, k}; \mathcal{G}_m, \beta)} \right)^2,
\end{equation}
where $1/h(x_{m,k})$ is to the uniform backward policy as  described in \cref{sec:policy}. Since the loss is computed over multiple instances $\mathcal{G}_m$, the objective corresponds to the conditional trajectory balance.

\begin{figure*}[t]
     \centering
     \includegraphics[width=1.0\textwidth]{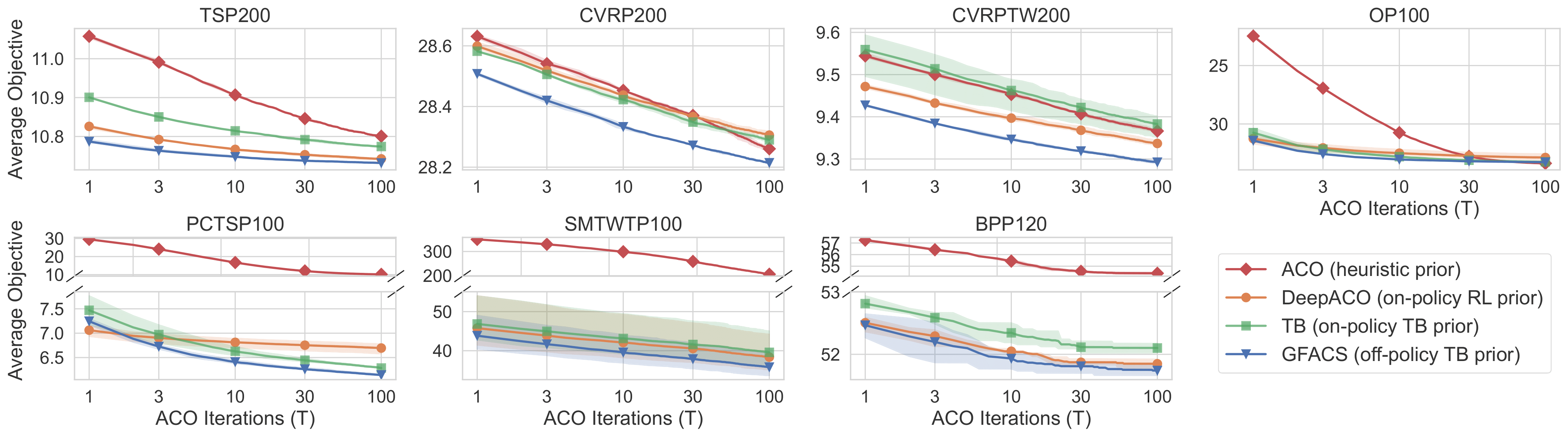}
    \vspace{-1.5em}
    \caption{Results of ACO algorithms with different priors on various CO tasks. Our GFACS outperforms every ACO baseline. The results are averaged over 3 independent models evaluated on the held-out test sets, and the shade indicates the min-max range of the 3 models.}
    \label{figure:aco-exp}
\end{figure*}

As discussed in~\cref{sec:gfn}, our target sampling distribution is proportional to reward, \textit{i.e.}, $p(x;\mathcal{G}) \propto R(x;\mathcal{G}, \beta)$. However, in most CO problems, including TSP, the scale of $R(x;\mathcal{G}, \beta)$ can vary significantly across instances, possibly leading to unstable training due to the high variance of logarithms of rewards. Similar to previous works that stabilize REINFORCE training \citep{kwon2020pomo,kim2022sym}, we normalize the energy $\mathcal{E}$ using the sample mean over $K$ samples for the same $\mathcal{G}_m$ as follows:
\begin{equation} \label{eq:shared}
    \tilde{\mathcal{E}}(x_{m, k};\mathcal{G}_m) = \mathcal{E}(x_{m, k};\mathcal{G}_m) - \frac{1}{K}\sum_{k=1}^{K} \mathcal{E}(x_{m, k};\mathcal{G}_m).
\end{equation}
Accordingly, we have $\tilde{R}(x_{m, k};\mathcal{G}_m, \beta) = e^{-\beta \tilde{\mathcal{E}}(x_{m, k};\mathcal{G}_m)}$. We call this technique a \textit{shared energy normalization}. Note that this technique is applied independently to the two batches, $\mathcal{B}_{\text{explore}}$ and $\mathcal{B}_{\text{exploit}}$, as the energy distribution between the two batches is often significantly different from each other, especially when a heuristic local search is applied. By applying shared energy normalization to both batches, we obtain $\tilde{\mathcal{B}}_{\text{explore}}$ and $\tilde{\mathcal{B}}_{\text{exploit}}$, in which every $\mathcal{E}$ is replaced by $\tilde{\mathcal{E}}$. 

The final loss to minimize is then defined as
\begin{equation}
    \mathcal{L}(\theta):= \frac{1}{2}L( \mathcal{B}_{\text{explore}}) + \frac{1}{2}L(\mathcal{B}_{\text{exploit}}), 
\end{equation}
where $L$ is as defined in Eq.(\ref{eq:tb_ours}). Note the parameters $\theta$ govern the functions $Z(\mathcal{G}_m)$ and $P_F(\cdot; \eta(\mathcal{G}_m))$, although this dependence has been omitted for simplicity. The gradient of the loss with respect to $\theta$ is given by $\nabla_{\theta}\mathcal{L}(\theta)$, and the parameters are optimized using stochastic gradient descent methods.

\section{Experiments}

This section presents experimental results to validate our algorithm. We compare GFACS against other ACO algorithms with different priors and also against competitive reinforcement learning algorithms for vehicle routing problems. Moreover, we conducted an ablation study and an empirical comparison with previous GFlowNets training methods. Last but not least, we validate the efficacy of multi-modal prior with another posterior update method beyond ACO, the active search (AS). Below, we provide a brief overview of the experimental setup.

\textbf{Implementation.} GFACS is built upon the implementation established by DeepACO~\citep{ye2023deepaco}. As our methodology mainly focuses on improving training algorithms rather than network architecture, we simply adopt the GNN architecture proposed in DeepACO, except the additional parameters for partition function $Z(\cdot;\theta)$ in Eq.(\ref{eq:tb_ours}). For detailed network architecture, hyperparameters, and hardware settings, please refer to \Cref{app:implement}. It is worth noting that one important hyperparameter that GFACS newly introduces is the inverse energy temperature $\beta$, which enables us to control the diversity-optimality trade-offs (see \Cref{sec:beta_tradeoff} for detailed analysis). In addition, we employ Ant System (\cref{sec:aco}) for the pheromone updates throughout the experiments. We provide experimental results with other ACO variants in \Cref{append:acovariants}.

\textbf{Test settings.} Each algorithm is tested on the held-out datasets, each with 100 instances generated in advance except TSP, where we obtain the TSP test dataset with 128 instances from \citet{ye2023deepaco}. Unless otherwise stated, the results presented throughout this section are the averaged values from three independent models, each trained with a distinct seed.

\begin{table}[t]
    \centering
    \caption{Experimental results on real-world datasets, TSPLib and CVRPLib. ‘$N$’ denotes the range of instance size and ‘\#’ denotes the number of instances in each set. We report the average optimality gap using the best-known costs.}
    \vspace{-0.5em}
    \resizebox{0.95\linewidth}{!}{



\begin{tabular}{clcccc}
\midrule
& \multicolumn{1}{c}{$N$} & \# & ACO & DeepACO & GFACS (ours) \\
\midrule
\parbox[t]{2mm}{\multirow{3}{*}{\rotatebox[origin=c]{90}{\scriptsize{TSPLib}}}} & 100-299 & 30& 1.71\%& 1.25\%& \textbf{1.21}\% \\
& 300-699 & 10 & 4.26\% &2.70\% & \textbf{2.60}\% \\
& 700-1499 & 12 &7.01\% &4.08\% & \textbf{3.88}\% \\
\midrule
\parbox[t]{2mm}{\multirow{3}{*}{\rotatebox[origin=c]{90}{\scriptsize CVRPLib}}} & 100-299 & 43 &\textbf{2.50}\% &3.27\% & 2.69\% \\
& 300-699 & 40 &3.75\% &4.22\% & \textbf{3.71}\% \\
& 700-1001 & 17 &4.64\% &5.07\% & \textbf{4.32}\% \\

\bottomrule
\end{tabular}

}
    \label{tab:lib}
    \vspace{-0.5em}
\end{table}

\begin{table*}[t]
    \centering
    \caption{Experimental results on TSP and CVRP. Results for methods with {*} are drawn from \citet{ye2023deepaco}, \citet{cheng2023select}, and \citet{jin2023pointerformer}. The reported objective value is averaged over 128 instances of TSP and 100 instances of CVRP. We report the average computation time per instance.
    }
    \vspace{-0.5em}
    \resizebox{1.0\textwidth}{!}
    {\begin{tabular}{lcccccclcccccc}
\cmidrule[1.0pt](lr{0.1em}){1-7}
\cmidrule[1.0pt](lr{0.1em}){8-14}
 & \multicolumn{3}{c}{TSP500} & \multicolumn{3}{c}{TSP1000} &
 & \multicolumn{3}{c}{CVRP500} & \multicolumn{3}{c}{CVRP1000} 
\\ 
\cmidrule[0.5pt](lr{0.2em}){2-4}
\cmidrule[0.5pt](lr{0.2em}){5-7}
\cmidrule[0.5pt](lr{0.2em}){9-11}
\cmidrule[0.5pt](lr{0.2em}){12-14}
Method & Obj. & Gap(\%) & Time
 & Obj. & Gap(\%) & Time &
Method & Obj. & Gap(\%) & Time
 & Obj. & Gap(\%) & Time 
\\ 
\cmidrule[1.0pt](lr{0.1em}){1-7}
\cmidrule[1.0pt](lr{0.1em}){8-14}
Concorde & 16.55 & - & 10.7s & 23.12 & - & 108s &
PyVRP & 62.96 & - & 21.0m & 119.29 & - & 1.4h 
\\
LKH-3 & 16.55 & 0.02 & 31.7s & 23.14 & 0.10 & 102s &
LKH-3 & 63.07 & 0.18 & 18.6m & 119.52 & 0.20 & 2.1h
\\
\cmidrule(lr{0.1em}){1-1}
\cmidrule(lr{0.1em}){2-4}
\cmidrule(lr{0.1em}){5-7}
\cmidrule(lr{0.1em}){8-8}
\cmidrule(lr{0.1em}){9-11}
\cmidrule(lr{0.1em}){12-14}
AM{*} & 21.46 & 29.07 & 0.6s & 33.55 & 45.10 & 1.5s 
\\
POMO{*} & 20.57 & 24.40 & 0.6s & 32.90 & 42.30 & 4.1s &
POMO & 68.69 & 9.08 & 2s & 145.74 & 22.14 & 3s
\\
\qquad+EAS & 18.25 & 10.29 & 300s & 30.77 & 33.06 & 600s & 
\qquad+EAS & 64.76 & 2.85 & 5.0m & 126.25 & 5.83 & 10.0m
\\
\qquad+SGBS & 18.86 & 13.98 & 318s & 28.58 & 23.61 & 605s &
\qquad+SGBS & 65.44 & 3.93 & 5.2m & 127.90 & 7.21 & 10.1m  
\\
DIMES{*} & 17.01 & 2.78 & 11s & 24.45 & 5.75 & 32s &
Sym-NCO  & 68.81 & 9.28 & 2s & 141.82 & 18.86 & 3s 
\\
SO{*} & 16.94 & 2.40 & 15s & 23.77 & 2.80 & 26s &
\qquad+EAS & 64.63 & 2.65 & 5.0m & 125.58 & 5.27 & 10.0m 
\\
Pointerformer{*} & 17.14 & 3.56 & 14s & 24.80 & 7.90 & 40s &
\qquad+SGBS  & 65.40 & 3.87 & 5.2m & 127.53 & 6.90 & 10.1m 
\\
\cmidrule(lr{0.1em}){1-1}
\cmidrule(lr{0.1em}){2-4}
\cmidrule(lr{0.1em}){5-7}
\cmidrule(lr{0.1em}){8-8}
\cmidrule(lr{0.1em}){9-11}
\cmidrule(lr{0.1em}){12-14}
ACO & 17.38 & 5.04 & 14s & 24.76 & 7.09 & 57s &
ACO & 64.62 & 2.63 & 36s & 122.82 & 2.96 & 1.3m 
\\
DeepACO & 16.84 & 1.77 & 15s & 23.78 & 2.87 & 66s &
DeepACO & 64.49 & 2.43 & 36s & 122.15 & 2.40 & 1.3m
\\ 
GFACS (ours) & \textbf{16.80} & \textbf{1.56} & 15s & \textbf{23.72} & \textbf{2.63} & 66s &
GFACS (ours) & \textbf{64.19} & \textbf{1.95} & 36s & \textbf{121.80} & \textbf{2.11} & 1.3m
\\
\cmidrule[1.0pt](lr{0.1em}){1-7}
\cmidrule[1.0pt](lr{0.1em}){8-14}
\end{tabular}
}
    \vspace{-0.5em}
    \label{tab:tsp_nls}
\end{table*}

\subsection{Comprison with different priors}
\label{sec:exp-vs-other-aco}

\textbf{Baselines.} We compare GFACS against the vanilla ACO (heuristic prior) and the DeepACO \citep{ye2023deepaco} (on-policy RL prior) and ACO with TB prior as baseline methods. We aim to show our multi-modal prior from our off-policy GFlowNets training methods is beneficial when this is integrated with the iterative pheromone update of ACO. We use $K=100$ ants and $T=100$ ACO iterations for all algorithms.

\textbf{Problems.} We evaluate our algorithm in a range of CO problems, including five routing problems -- TSP, Capacitated Vehicle Routing Problem (CVRP), CVRP with Time Window (CVRPTW), Prize-Collecting TSP (PCTSP), and Orienteering Problem (OP) --, a scheduling problem, the Single Machine Total Weighted Tardiness Problem (SMTWTP), and a grouping problem, the Bin Packing Problem (BPP). Detailed descriptions and configurations for each problem can be found in~\cref{app:problems}. For TSP, CVRP and CVRPTW, we use local search heuristics (2-opt, Swap* and \texttt{LocalSearch} operator in PyVRP \citep{Wouda_Lan_Kool_PyVRP_2024}, respectively) with perturbation scheme similar to \citet{ye2023deepaco} for training GFACS and testing all the ACO algorithms. For others, we use the destroy-and-repair method. In \cref{append:destroy-and-repair}, further details are provided.

\textbf{Results.} 
\cref{figure:aco-exp} shows GFACS consistently outperforms both the classic ACO algorithm and DeepACO across all benchmark problems. Notably, for routing problems, GFACS surpasses DeepACO by a significant margin. This is particularly remarkable given that both algorithms use identical neural networks and local search operators, underscoring the importance of the multi-modal prior. Additionally, the proposed off-policy training method enhances GFACS performance, demonstrating its effectiveness for GFlowNets.

\textbf{Evaluation on real-world datasets.} 
We compare the performance of our model against baseline models on real-world TSP and CVRP instances from the TSPLib \citep{reinelt1991tsplib} and CVRPLib X-instances \citep{uchoa2017new}. The models were trained on random uniform instances of TSP/CVRP with sizes 200 for dataset sizes 100-299, 500 for sizes 300-699, and 1000 for sizes 700-1499, with all other settings remaining the same. \Cref{tab:lib} presents the results. GFACS outperforms the baselines in TSPLib and large CVRPLib instances but slightly underperforms compared to vanilla ACO on small CVRPLib instances, where ACO with Swap* local search suffices for small-scale CVRP. As the problem scale increases, ACO's performance declines, underscoring the importance of effective priors using GFlowNets for efficient pheromone posterior updates in large combinatorial spaces.

\subsection{Comparison with RL solvers}
\label{sec:exp-vs-nco}

In this section, we present experimental results comparing our method with problem-specific RL approaches that leveraged specialized neural architectures or learning heuristics for the TSP and CVRP.

\textbf{Baselines.} For TSP baselines, we select the AM \citep{kool2018attention}, POMO \citep{kwon2020pomo}, Sym-NCO \citep{kim2022sym} and Pointerformer \citep{jin2023pointerformer} as RL-based constructive methods, DIMES \citep{qiu2022dimes} as RL-based heatmap-based methods, the Select and Optimize \citep[SO; ][]{cheng2023select} as RL-based improvement methods. We also applied two search algorithms, the Efficient Active Search \citep[EAS; ][]{hottung2021efficient} and Simulation Guided Beam Search \citep[SGBS; ][]{choo2022simulation}, to POMO for a more thorough comparison. For CVRP,\footnote{Due to the inaccessibility of source codes, we compare ours with TAM using the reported result in the paper \citep{hou2022generalize}. See \cref{app:tam}.} we employ POMO and Sym-NCO as baselines with EAS and SGBS. Note that we restrict the duration of these searches to no more than 5 and 10 minutes for problems with 500 and 1000 nodes, respectively. We use Concorde \citep{concorde} for TSP and PyVRP \citep{Wouda_Lan_Kool_PyVRP_2024} for CVRP as oracle heuristic solvers to obtain near-optimal solutions that serve as baselines with $0.00\%$ gap. For ACO algorithms, we use $K = 100$ ants with $T = 10$ with local search algorithms similar to \citet{ye2023deepaco}.

\textbf{Results.} GFACS consistently shows superior or highly competitive performance compared to evaluated baseline methods (\cref{tab:tsp_nls}). Ours outperforms fast zero-shot methods like AM and POMO—even when they employ search techniques such as EAS or SGBS—and achieves better results than similarly paced algorithms like DIMES and Pointformer. Although it performs slightly worse than the heuristic solver LKH-3 \citep{lkh2017}, GFACS offers faster inference times, especially in CVRP. As a general-purpose algorithm, it holds significant appeal in the competitive benchmarks of both TSP and CVRP. Notably, despite its small neural architecture and requiring less than three hours of training, GFACS outperforms larger models like POMO and Sym-NCO, which need over a week of training \citep{kim2022sym}.

\subsection{Ablation study} \label{sec:ablation}

\begin{table}[t]
    \centering
    \caption{Results of the ablation study.}
    \vspace{-0.5em}
    \resizebox{0.9\linewidth}{!}{



\begin{tabular}{lccc}
\toprule[1.0pt]
Method & TSP500 & CVRP500 & BPP120
\\
\midrule
GFACS & \textbf{16.80} & \textbf{64.19} & \textbf{51.84} \\
\quad -- \textit{Off-policy} & 17.09 & 64.42 & 52.32 \\
\quad -- \textit{Energy Reshaping} & 16.87 & 64.41 & 52.07 \\
\quad -- \textit{Shared Energy Norm.} & 17.52 & 65.51 & 90.73 \\
\quad\quad + \textit{VarGrad} & 16.86 & 64.28 & 51.95\\
\bottomrule[1.0pt]
\end{tabular}


} 
    \vspace{-0.5em}
    \label{tab:ablation}
\end{table}

This section reports the ablation study results, validating our proposed training techniques for stabilizing and enhancing GFlowNets in large-scale combinatorial optimization tasks. As detailed in \cref{tab:ablation}, our off-policy training method, incorporating local search, energy reshaping to value potential solutions, and shared $\epsilon$-norm to stabilize conditional GFlowNets training, significantly improves performance across three tasks: TSP, CVRP, and BPP. For comprehensive results, including standard deviations, refer to \cref{app:ablation}.

\subsection{Comparison with other GFlowNet-based training methods}
\label{sec:exp-fl-db}

While this paper proposes a novel training method for GFlowNets designed for large-scale combinatorial optimization problems that are not locally decomposable due to their cyclic nature, such as in TSP and CVRP, we also compare existing GFlowNets training methods within this domain.

\begin{figure}[t]
    \centering
    \vspace{-0.25em}
    \includegraphics[width=0.82\linewidth]{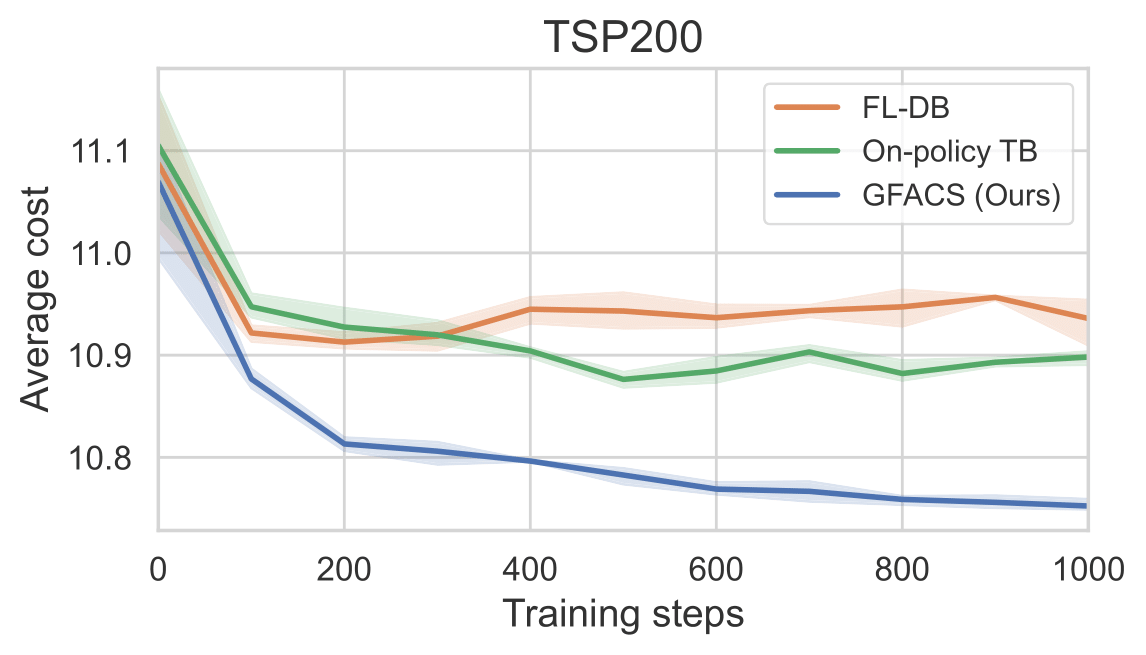}
    \vspace{-0.5em}
    \caption{Validation cost on TSP with 200 nodes during training, compared to forward-looking detailed balance (FL-DB).}
    \vspace{-0.75em}
    \label{fig:vsFLDB}
\end{figure}

\textbf{Forward-looking detailed balance} \citep[\textbf{FL-DB};][]{pan2023better, zhang2023let}. FL-DB leverages local credits when the application benefits from useful local rewards for partial solutions. While graph combinatorial optimization tasks are generally locally decomposable \citep{ahn2020learning}, yielding useful local rewards, tasks like TSP do not offer substantial local rewards due to their cyclic nature, which imposes a global constraint, thus complicating local credit assignment. As shown in \cref{fig:vsFLDB}, various GFlowNet methods were used to train the ACO's heuristic matrix $\eta(\mathcal{G})$. While FL-DB did not surpass on-policy TB for TSP, our off-policy method demonstrated significant improvement. This trend holds across different deep learning models and inference tasks beyond ACO; see \cref{append:fl-db}.

\textbf{Gradient of log-partition variance (VarGrad).} Originally proposed by \citet{richter2020vargrad} for variational inference, the VarGrad loss was applied to conditional GFlowNets for scheduling problems by \citet{zhang2022robust}. This approach is vital for combinatorial optimization tasks requiring conditional inference on the problem instance $\mathcal{G}$. It computes the log-partition function via the consistency property over trajectories conditioned on identical instances, thus bypassing explicit training of $\log Z(\mathcal{G})$. Explicit training is challenging due to varying energy landscapes across different instances. Our method, utilizing explicit training of the log-partition $\log Z(\mathcal{G})$ as in Eq.(\ref{eq:tb_ours}), outperforms VarGrad in training conditional GFlowNets (\cref{tab:ablation}). This improvement stems from our \textit{Shared energy normalization}, which stabilizes training by normalizing energy values to sample means, resulting in an energy distribution with a zero mean across instances. Note that accurately estimating the explicit log-partition value is known to stabilize training of GFlowNets \citep{malkin2022gflownets}.

\subsection{Integrating GFACS with active search}
\label{sec:exp-as}

\textbf{Background.} In this section, we explore the integration of active search (AS) as a posterior update on a parameterized prior distribution \( p_{\theta}(\cdot|\mathcal{G}) \). In AS, the prior parameter \(\theta\), pretrained on a large set of random problem instances \(\mathcal{G}\), is adaptively updated for specific test-time instances \(\mathcal{G}_{\text{test}}\) using the same objective function as in the prior training. We trained this prior \( p_{\theta} \) using two methods: on-policy RL method \citep{kwon2020pomo} that employ REINFORCE with shared baselines, and our off-policy GFlowNet training method as explained in~\cref{sec:training}.

\textbf{Setting.} For the active search procedure, we sample 20 solutions per test-time instance and perform 5 gradient updates at each iteration, updating \(\theta\) over a total of 50 iterations.

\begin{table}[t]
    \centering
    \caption{Evaluation results combine with active search (AS) as a posterior search method.}
    \vspace{-0.5em}
    \resizebox{1.0\linewidth}{!}{\begin{tabular}{lcccc}
\toprule[1.0pt]
\multicolumn{1}{c}{} & \multicolumn{2}{c}{TSP200} & \multicolumn{2}{c}{CVRP200}\\
\cmidrule[0.5pt](lr{0.2em}){2-3}
\cmidrule[0.5pt](lr{0.2em}){4-5}
Method & Obj. & Gap (\%) & Obj. & Gap (\%) \\
\midrule
Concorde / PyVRP & 10.72 & 0.00 & 27.93 & 0.00 \\
\midrule
RL prior + AS & 10.79  \footnotesize{$\pm$ 0.00} & 0.65 & 28.43 \footnotesize{$\pm$ 0.02} & 1.79 \\
GFlowNet prior + AS & \textbf{10.75}  \footnotesize{$\pm$ 0.00} & \textbf{0.29} & \textbf{28.34} \footnotesize{$\pm$ 0.01} & \textbf{1.47} \\
\bottomrule
\end{tabular}}
    \label{tab:active_search}
    \vspace{-0.75em}
\end{table}

\textbf{Results.} As shown in \Cref{tab:active_search}, the off-policy GFlowNet prior combined with AS posterior inference outperforms the RL counterpart.  This result provides further evidence supporting our hypothesis that the multi-modal prior from GFlowNet training synergizes effectively with the subsequent iterative posterior update algorithm. Additionally, the results demonstrate the flexibility of our method beyond the ACO, suggesting a promising direction for future research into the benefits of using a multi-modal informative prior.

\section{Conclusion}

In this paper, we introduced a novel probabilistic meta-heuristic that hierarchically combines GFlowNets-based amortized multi-modal priors with posterior search using ant colony optimization (ACO) for combinatorial optimization problems. Our strategy achieved remarkable success across seven benchmark combinatorial optimization tasks. Furthermore, by applying our meta-heuristic to other posterior search methods like active search (AS), we demonstrated the expandability and modularity of our approach. Limitation of our work is that, while our method empirically improves solution quality, we do not provide formal theoretical guarantees in how our method gives better optimality than the reinforcement learning counterpart.

\section*{Acknowledgment}

The authors thank Haoran Ye and Federico Berto for valuable discussions for this project. This research is based on funding from Samsung, Intel, CIFAR and the CIFAR AI Chair program. The research was enabled in part by computational resources provided by the Digital Research Alliance of
Canada (\url{https://alliancecan.ca}), Mila (\url{https://mila.quebec}), and NVIDIA. This work was also partially supported by the National Research Foundation of Korea (NRF) grant funded by the
Korea government (MSIT) (No. RS-2024-00410082).

\bibliography{reference}

\appendix
\onecolumn

\section{Implementation} \label{app:implement}

\subsection{Learning architecture and optimizer}

We utilize the anisotropic graph neural network with edge gating mechanisms, which is proposed by \citet{joshi2021learning} and adopted by \citet{ye2023deepaco}, as the foundational GNN architecture for the construction of a heuristic matrix. This network offers node and edge embeddings that are more beneficial for edge prediction than other popular GNN variants, such as Graph Attention Networks \citep{velivckovic2018graph} or Graph Convolutional Networks \citep{kipf2016semi}.

We represent the node feature of node $v_i$ at the $\ell$-th layer as $h_i^\ell$, and the feature of the edge between nodes $i$ and $j$ at the same layer as $e_{ij}^\ell$. These features undergo propagation to the next layer $\ell+1$ through an anisotropic message passing scheme, detailed as follows:

\begin{align}
& h_i^{\ell+1} =  h_i^\ell + \text{ACT}(\text{BN}(U^{\ell}h_i^\ell + \mathcal{A}_{j\in \mathcal{N}_i}(\sigma(e^\ell_{ij}) \odot V^\ell h_j^\ell))) \\
& m^\ell = P^\ell e_{ij}^\ell + Q^\ell h_i^\ell + R^\ell h_j^\ell \\
& e_{ij}^{\ell+1} = e_{ij}^\ell + \text{ACT}(\text{BN}(m^\ell))
\end{align}

where $U^\ell, V^\ell, P^\ell, Q^\ell, R^\ell \in R^{d \times d}$ are the learnable parameters of layer $\ell$, $\text{ACT}$ denotes the activation function, BN \citep{ioffe2015batch} denotes batch normalization, $\mathcal{A}$ denote the aggregation function, $\sigma$ is the sigmoid function, $\odot$ is Hadamard product, $\mathcal{N}_i$ denotes the neighborhoods of node $v_i$. In this paper we use $\text{ACT}$ is SiLU \citep{elfwing2018sigmoid} and $\mathcal{A}$ is mean pooling. We transform the extracted edge features into real-valued heuristic metrics by employing a 3-layer Multilayer Perceptron (MLP) that includes skipping connections to the embedded node feature $h$. The SiLU is the activation function for all layers except the output layer, in which the sigmoid function is applied to generate normalized outputs. The training process is optimized using the AdamW optimizer \citep{loshchilov2017decoupled} in conjunction with a Cosine Annealing Scheduler \citep{loshchilov2016sgdr}

\subsection{Hyperparameters}\label{sec:hyperparams}

\begin{table}[ht]
    \centering
    \vspace{-5pt}
    \caption{Training hyperparameter settings for each task. These hyperparameter settings are used both for GFACS and DeepACO training. The values with * indicate that the values are used for fine-tuning a pre-trained model.}
    \resizebox{1.0\linewidth}{!}{\begin{tabular}{l|ccc|ccc|c|c|c|c}
\toprule
 & TSP200 & TSP500 & TSP1000 & CVRP200 & CVRP500 & CVRP1000 & OP100 & PCTSP100 & SMTWTP100 & BPP120 \\
\midrule
$N_{\text{epoch}}$ & 50 & 50 & 20* & 50 & 50 & 20* & 20 & 20 & 20 & 20 \\
$N_{\text{inst}}$ & 400 & 400 & 200* & 200 & 200 & 100* & 100 & 100 & 100 & 100\\
$M$ & 20 & 20 & 10 & 10 & 10 & 5 & 5 & 5 & 5 & 5 \\
$K$ & 30 & 30 & 20 & 20 & 20 & 15 & 30 & 30 & 30 & 30 \\
\bottomrule
\end{tabular}}
    \label{tab:hyperparams_train}
    \vspace{0pt}
\end{table}

For a task of COP (e.g., TSP200), we train GFACS for $N_{\text{epoch}}$ epochs. In each epoch, we use $N_{\text{inst}}$ number of instances using mini-batch with size $M$; thus $N_{\text{inst}} / M$ gradient steps are taken per epoch. For each problem instance, we sample $K$ solutions to train with, as discussed in \cref{sec:stepA}. Our main baseline algorithm, DeepACO, is also trained using the same hyperparameters to ensure a fair comparison. We summarize used value for $N_{\text{epoch}}$, $N_{\text{inst}}$, $M$ and $K$ in \cref{tab:hyperparams_train}. Note that for TSP1000 and CVRP1000, we use model checkpoints pre-trained in TSP500 and CVRP500, respectively, to reduce the processing time for training.

There are other hyperparameters to be considered specifically for GFACS. $\beta$ is one of the most important hyperparameters for training GFACS, which is analyzed in \cref{sec:beta_tradeoff}. We use annealing for $\beta$ using the following log-shaped scheduling:

\begin{equation*}
    \beta = \beta_{\text{min}} + (\beta_{\text{max}} - \beta_{\text{min}})\min\left(\frac{\log(i_{\text{epoch}})}{\log(N_{\text{epoch}} - N_{\text{flat}})}, 1\right)
\end{equation*}

where $i_{\text{epoch}}$ is the index of epoch and $N_{\text{flat}}$ is the number of last epochs to have $\beta = \beta_{\text{max}}$.

\begin{table}[ht]
    \centering
    \caption{Settings of $\beta$ for each task.}
    \resizebox{0.6\linewidth}{!}{\begin{tabular}{lccccccc}
\toprule
 & TSP & CVRP & CVRPTW & OP & PCTSP  & SMTWTP & BPP \\
\midrule
$\beta_{\text{min}}$ & 200 & 500 & 500 & 5 & 20 & 10 & 1000\\
$\beta_{\text{max}}$ & 1000 & 2000 & 2000 & 10 & 50 & 20 & 2000\\
\bottomrule
\end{tabular}}
    \label{tab:hyperparams_beta}
    \vspace{0pt}
\end{table}

We report the $\beta$ value used for each task in \cref{tab:hyperparams_beta}. Note that $\beta$ remains consistent across different scales of the same problem class.

In \Cref{eq:energy_reshaping}, we also need a hyperparameter $\alpha$ that controls how much we augment the energy of a solution with the energy of the refined solution. If $\alpha = 0$, we do not augment at all, and if $\alpha = 1$, the energy of a solution is defined to be equal to the energy of the refined solution. We empirically found that linearly increasing $\alpha$ from $0.5$ to $1.0$ along the epochs works well, and we use that setting for all experiments.

\subsection{Computing resource}\label{append:compute-resources}
Throughout the experiments, we utilize a single CPU, specifically an AMD EPYC 7542 32-Core Processor, and a single GPU, the NVIDIA RTX A6000, for training neural networks. After thet, the inference was conducted on a separate, standalone PC with a single AMD Ryzen 9 5900X 12-Core Processor and NVIDIA RTX 4070 Ti to assess the wall time of each algorithm fairly.

\section{Problem settings} \label{app:problems}

In this section, we provide an overview of the combinatorial optimization problems we address. These encompass vehicle routing problems, bin packing problems, and scheduling problems. For a comprehensive understanding of each problem's specific settings, we direct readers to the work by \citet{ye2023deepaco}. Furthermore, for access to the exact source code pertaining to each problem, please refer to the repository available at \footnote{\url{https://github.com/henry-yeh/DeepACO}} from \citet{ye2023deepaco}.

\paragraph{Traveling salesman problem} 
In the traveling salesman problem (TSP), we aim to find the shortest route that visits a set of nodes once and returns to the origin. We consider Euclidean TSP, where the distance between two nodes is determined by Euclidean distance. A TSP instance can be determined as a set of nodes with 2D coordinates. We generate a random TSP instance by sampling the coordinate from a unit square, $[0, 1]^2$.

\paragraph{Capacitated vehicle routing problem}
Similar to TSP, capacitated vehicle routing problem (CVRP) also seeks the shortest path that visits all the nodes once. In CVRP, however, a vehicle has a limited capacity, and each node has a demand to be satisfied. The sum of the demand of each subroute should not exceed the capacity, which limits the number of nodes that can be visited by one vehicle. In our setting, CVRP is also defined in Euclidean space.

A CVRP instance can be represented by a set of customer nodes, a single depot node, and a capacity $C$. Each customer node has a position (2D coordinates) and a demand, while the depot only has a position. A random CVRP instance is generated by sampling the coordinates of both customer and depot from a unit square $[0, 1]^2$ and sampling demands for each customer from a predefined uniform distribution, $U[a, b]$. We use $a=1$ and $b=9$, and a fixed $C=50$ for all problem scales, 200, 500 and 1000.

\paragraph{Capacitated vehicle routing problem with time window} A uniform-capacity fleet of delivery vehicles is tasked with servicing customers who have specific demand and operating hours (earliest-latest arrival) for a single commodity. The vehicles should deliver to customer only the time available for service. The goals are to minimize the total travel distance.

A CVRPTW instance consists of a set of customer nodes, a single depot node, the demands of each customer node, the vehicle capacity (as in CVRP), and additionally, the earliest and latest arrival times for each node, as well as the service time at each node. We generate the CVRPTW instance following the method provided by \citet{berto2023rl4co}, which can be accessed here.\footnote{\url{https://github.com/ai4co/rl4co/blob/main/rl4co/envs/routing/cvrptw/generator.py}}

\paragraph{Orienteering problem}
In the orienteering problem (OP), the objective is to maximize the total prize of visited nodes within a given tour length limitation. This problem is similar to CVRP but focuses on maximizing gains rather than minimizing costs.

An OP instance comprises a set of nodes to visit, each with its position and prize, a depot, and a maximum length constraint. We sample a random OP instance by first sampling the position of nodes, including the starting node, from a unit square $[0, 1]^2$. And we set the prize for each node, following \cite{fischetti1998solving,ye2023deepaco}, to be $p_i = (1 + \lfloor 99 \cdot \frac{d_{0,i}}{\max_{j=1}^{n}{d_{0,j}}} \rfloor) / 100$, where $d_{0,i}$ is the Euclidean distance from the depot to node $i$. We set the maximum length to 4 in the OP100 task, following \citet{kool2018attention}.

\paragraph{Prize-collecting traveling salesman problem}
The prize-collecting traveling salesman problem (PCTSP) extends the TSP by associating a prize for visiting a node and a penalty for not visiting it. The goal is to find a tour that minimizes the total travel cost plus penalties for the unvisited nodes and also satisfies a predefined prize threshold. In PCTSP, balancing the need to collect prizes against the cost of visiting more nodes matters.

A PCTSP instance comprises a set of nodes to visit, a depot, and a minimum total prize threshold. The coordinates of nodes and depot are sampled in the unit square. The prize for each node is sampled from $U[0,1]$, while the penalty for each node is sampled from $U[0, C_N]$, where $C_N$ is set to $0.12$ for $N=100$. Lastly, the minimum total prize threshold is $N/4$ for problems with ${N}$ nodes, i.e., $25$ for PCTSP100.

\paragraph{Single machine total weighted tardiness problem}
This problem involves scheduling jobs on a single machine where each job has a due date, processing time, and weight. The objective is to minimize the total weighted tardiness, which is the sum of the weights of the jobs multiplied by their tardiness (the amount of time a job is completed after its due date).

To generate an SMTWTP instance, we sample a due date from $U[0, N]$, the weight from $U[0, 1]$, and the processing time from $U[0, 2]$, where $N$ is the number of jobs to process.

\paragraph{Bin packing problem}
The bin packing problem (BPP) aims to minimize the number of bins when objects with different volumes need to be packed into a finite number of bins. The cost functions are often modified to maximize the utility of used bins \citep{falkenauer1992bpp, ye2023deepaco}.
However, the direct use of the objective function employed in DeepACO can lead to an excessively large $\beta$ since the objective values tend to be saturated near 1.0; note that the objective values can not exceed 1.0 in that setting. Therefore, we slightly modify them to minimize a weighted sum of the inverse of the utilization of the used bins throughout all our BPP experiments.

To generate a BPP instance, we sample the $N$ items with a random size from $U[20, 100]$. And we use 150 for the bin capacity.

\section{Implementation of destroy-and-repair local search}\label{append:destroy-and-repair}

As demonstrated in \cref{sec:training}, our algorithm utilizes the local search for off-policy training. The refined solutions from local search serve as exploitation batches, and they are also used for energy-reshaping (\Cref{eq:energy_reshaping}).
 For some problems like traveling salesman problem (TSP), capacitated vehicle routing problem (CVRP), and CVRP with time window (CVRPTW), we used the well-established local search operators such as 2-opt, Swap* \citep{vidal2022hybrid} and granular neighborhood search from PyVRP \citep{Wouda_Lan_Kool_PyVRP_2024}.

While domain-specific local search techniques are well developed for many combinatorial optimization problems, there are still tasks that lack effective local search solvers like Swap*. A potential candidate for a general-purpose local search solver in combinatorial optimization is the destroy-and-repair local search operator. This method is inspired by the large neighborhood search approach described by \citet{pisinger2019large}. Our implementation utilizes an evolutionary-style algorithm that repeatedly performs destruction, reconstruction, and top-K selection over a specified number of rounds.

\textbf{Symmetricity-aware destruction.} The most naive way to destroy a solution $x$ is simply dropping $N_{destroy}$ of the last nodes from the sampled trajectory $\tau_{\to x}$. However, it only performs the destruction with one trajectory, while there can be multiple trajectories that correspond to the same solution $x$. Considering this symmetric nature, our destruction phase consists of 1) selecting $\tau'_{\to x}$ randomly from the set of symmetric trajectories and then 2) destroying $\tau'_{x}$ by dropping the last $N_{destroy}$ nodes, resulting in a partial solution. Note that the number of symmetric trajectories for a given solution differs from problem types, as discussed in \cref{app:symmetric}.

\textbf{Repair with temperature annealing.} Given a partial solution, we reconstruct a new solution $x'$ using the policy similar to the decoding steps explained in \cref{sec:policy}. Considering the purpose of the local search (refining a given solution), we further promote the exploitation by introducing an inverse temperature for the policy. Formally, we modify \cref{eq:policy} with the inverse temperature $\delta_r$ as follows:
\begin{equation}
    P_F(s_{t+1}|s_t; \eta(\mathcal{G};\theta), \rho, \delta_r)
    \propto \begin{cases} \left( \rho_{a_{t-1},a_{t}}\eta_{a_{t-1},a_{t}}(\mathcal{G}; \theta)\right)^{\delta_r} & \text{if $a_t \notin s_{t}$ }\\ 0 & \text{otherwise} \end{cases}
\end{equation}
We increase $\delta_r$ linearly from 1.0 to a specified maximum value, $\delta_{\text{max}}$, as the local search round progresses.

\textbf{Top-K selection}. For a given solution $x$, we perform the aforementioned destruction and repair multiple times (in batch), obtaining a set of refined solutions. Then we evaluate them and leave only top-K solutions among them in terms of solution quality. The remaining solutions are used as an initial solution for the next round of the search.

\section{Additional experiments}

\subsection{Comparison between trajectory balance and forward-looking detailed balance}
\label{append:fl-db}

\begin{table}[ht]
    \centering
    \vspace{-5pt}
    \caption{Comparison between forward-looking detailed balance (FL-DB) and trajectory balance (TB) loss in TSP50, using Attention Model (AM) and heatmap-based method from \citet{joshi2021learning} as backbone policies.}
    \resizebox{0.5\linewidth}{!}{\begin{tabular}{lcc}
\toprule[1.0pt]
\multicolumn{1}{c}{} & \multicolumn{2}{c}{TSP50}\\
\cmidrule[0.5pt](lr{0.2em}){2-3}
Method & Obj. & Gap (\%)\\
\midrule
Concorde & 5.69 & 0.00 \\
\midrule
AM + FL-DB ($\text{s.w.}=1280$) & 6.52 & 14.64 \\
AM + TB ($\text{s.w.}=1280$) & 6.00 & 5.50 \\
\midrule
Joshi et al. + FL-DB ($\text{s.w.}=1280$) & 6.54 & 14.96 \\
Joshi et al. + TB ($\text{s.w.}=1280$) & 6.09 & 7.05 \\
\bottomrule
\end{tabular}
}
    \label{tab:vsLTFT}
    \vspace{0pt}
\end{table}

We trained Attention Model~\citep[AM;][]{kool2018attention} and heatmap-based policy from \citet{joshi2021learning} in an unsupervised way using FL-DB and TB loss. \Cref{tab:vsLTFT} shows the test results for 128 random TSP50 instances. The result is aligned with the previous observation from GFACS, showing the limitation of FL-DB loss for solving complex CO tasks.

\subsection{Analysis of $\beta$ and diversity-optimality trade-off} \label{sec:beta_tradeoff}

In this subsection, we explored the diversity-optimality trade-off by varying the inverse energy temperature $\beta$ from low to high values. For analysis, we generate $K$ solution samples from policies trained with GFACS with different values of $\beta$ and DeepACO (with $\rho=1$ so that the pheromone does not affect the result).

We use average pairwise Jaccard distance as a diversity metric. Formally, the diversity of a set of $K$ solutions is calculated as $\frac{1}{K(K-1)}\sum_{i}\sum_{j} d_{J}(x_i, x_j)$, where $d_J$ is a Jaccard distance between two solutions, defined as follows:
\begin{equation}
    d_J(x_i, x_j) = 1 - J(x_i, x_j) = 1 - \cfrac{\vert E(x_i) \cap E(x_j) \vert}{\vert 
E(x_i) \cup E(x_j) \vert}
\end{equation} 
$E(x)$ is the set of edges composing the solution $x$. In TSP or CVRP, the (undirected) edge is a set of adjoining nodes in the solution path. For some CO problems, such as SMTWTP, the edge should be directed due to its asymmetric nature.

\begin{wrapfigure}{r}{0.5\textwidth}
    \centering
    \includegraphics[width=0.95\linewidth]{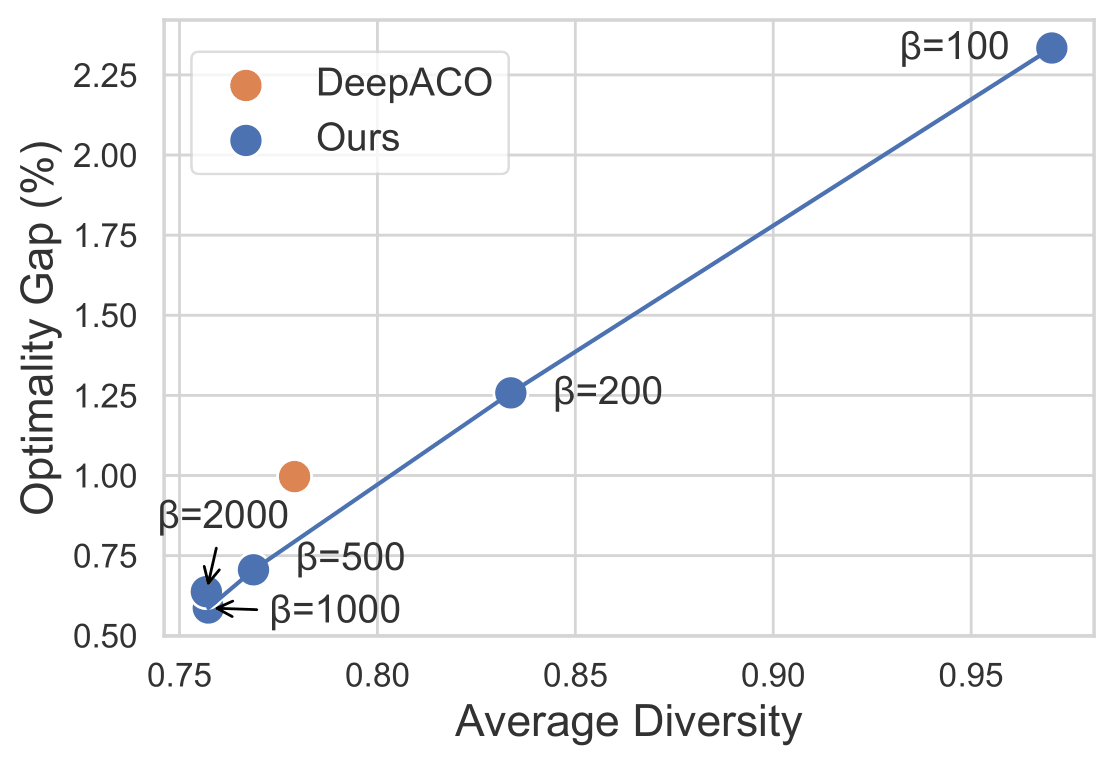}
    \vspace{-5pt}
    \caption{The trade-off between diversity and optimality as the inverse energy temperature $\beta$ changes.}
    \label{fig:diversity_tradeoff}
\end{wrapfigure}

\cref{fig:diversity_tradeoff} show that lower $\beta$ values lead to more diverse solutions with higher costs, while higher $\beta$ values result in less diverse solutions with lower costs. GFACS confirms the anticipated trade-off between diversity and optimality, establishing a Pareto frontier over DeepACO by offering more diverse and high-quality solutions. This advantage is particularly beneficial when an effective local search mechanism is available, increasing the likelihood of achieving superior local optima. However, excessively high $\beta$ (e.g., $2000$ in our result) degrades both diversity and solution quality, supposedly because too much high $\beta$ hinders exploration during training. This empirical evidence supports our design choice of $\beta$ annealing \cref{sec:hyperparams}, which encourages exploration in the early stage of training.

\subsection{Evaluation with other ACO variants}
\label{append:acovariants}

\begin{table}[ht]
    \centering
    \caption{Evaluation results with different pheromone update rules, Elitist Ant System (Elitist) and MAX-MIN Ant System (MAX-MIN). We run $T=10$ ACO rounds with 100 ants in all settings.}
    \vspace{-5pt}
    \resizebox{0.5\linewidth}{!}{\begin{tabular}{lcccc}
\toprule[1.0pt]
\multicolumn{1}{c}{} & \multicolumn{2}{c}{TSP200} & \multicolumn{2}{c}{CVRP200}\\
\cmidrule[0.5pt](lr{0.2em}){2-3}
\cmidrule[0.5pt](lr{0.2em}){4-5}
Method & Obj. & Gap (\%) & Obj. & Gap (\%) \\
\midrule
Concorde / PyVRP & 10.72 & 0.00 & 27.93 & 0.00 \\
\midrule
Ant System & 10.91 & 1.77 & 28.45 & 1.86 \\
DeepACO + Ant System & 10.77 & 0.45 & 28.44 & 1.80 \\
GFACS + Ant System & \textbf{10.75} & \textbf{0.26} & \textbf{28.32} & \textbf{1.40} \\
\midrule
Elitist & 10.97 & 2.31 & 28.46 & 1.90 \\
DeepACO + Elitist & 10.77 & 0.49 & 28.45 & 1.84 \\
GFACS + Elitist & \textbf{10.75} & \textbf{0.28} & \textbf{28.34} & \textbf{1.47} \\
\midrule
MAX-MIN & 10.97 & 2.34 & 28.46 & 1.88 \\
DeepACO + MAX-MIN & 10.78 & 0.57 & 28.44 & 1.82 \\
GFACS + MAX-MIN & \textbf{10.75} & \textbf{0.32} & \textbf{28.33} & \textbf{1.43} \\
\bottomrule
\end{tabular}}
    \label{tab:acovariants}
\end{table}

Our method is focused on enhancing the design of the heuristic matrix within the ACO framework. Thus, the advancement of the ACO algorithm itself (usually in pheromone update and action selection rule) is orthogonal to our work.
To validate this claim, we evaluate our method with different ACO variants other than Ant System, namely Elitist Ant System \citep{ant-sys}, MAX-MIN Ant System \citep{max-min-ant}.
The results in \Cref{tab:acovariants} demonstrate that GFACS consistently outperforms both ACO with human-designed heuristic and DeepACO under all pheromone update rules tested. This underscores the superiority of GFACS in generating a high-quality heuristic matrix.

\subsection{Comparison against TAM} \label{app:tam}

\begin{table}[ht]
    \centering
    \caption{Results for CVRP instances used in TAM. The results for AM and TAM is drawn from \citet{hou2022generalize}, and the result for DeepACO is drawn from \citet{ye2023deepaco}.}
    \vspace{-5pt}
    \resizebox{0.5\linewidth}{!}{\begin{tabular}{lcccc}
\toprule[1.0pt]
\multicolumn{1}{c}{} & \multicolumn{2}{c}{CVRP100} & \multicolumn{2}{c}{CVRP400}\\
\cmidrule[0.5pt](lr{0.2em}){2-3}
\cmidrule[0.5pt](lr{0.2em}){4-5}
Method & Obj. & Time (s) & Obj. & Time (s)\\
\midrule
AM* & 16.42 & 0.06 & 29.33 & 0.20 \\
TAM* & 16.08 & 0.09 & 25.93 & 1.35 \\
\midrule
DeepACO ($T=4$) & 16.08 & 2.97 & 25.31 & 3.65 \\
DeepACO ($T=10$) & 15.77 & 3.87 & 25.27 & 5.89 \\
\midrule
GFACS ($T=4$) & 15.77 & 2.91 & 24.95 & 3.23 \\
GFACS ($T=10$) & 15.72 & 3.88 & 24.87 & 5.21 \\
\bottomrule

\end{tabular}}
    \label{tab:tam}
\end{table}

We compare GFACS and DeepACO against CVRP instances used in Two-stage Divide Method~\citep[TAM;][]{hou2022generalize}. These instances have a different capacity constraint from those we used for the main experiment (\cref{tab:tsp_nls}). Please refer to \cite{hou2022generalize} for more details regarding the instance generation. Moreover, to match the processing time to be comparative to DeepACO, we set $K$ (the number of ants for ACO) to 16.

\subsection{Full results of the ablation study}
\label{app:ablation}

\begin{table}[ht]
    \centering
    \caption{Results of the ablation study. The mean and standard deviation are reported based on three independent runs. The reported gap is calculated against Concorde \citep{concorde} for TSP and PyVRP \citep{Wouda_Lan_Kool_PyVRP_2024} for CVRP.}
    \vspace{-5pt}
    \resizebox{0.95\linewidth}{!}{
\begin{tabular}{lccccc}
\toprule[1.0pt]
\multicolumn{1}{c}{} & \multicolumn{2}{c}{TSP500} & \multicolumn{2}{c}{CVRP500} & \multicolumn{1}{c}{BPP120}
\\
\cmidrule[0.5pt](lr{0.2em}){2-3}
\cmidrule[0.5pt](lr{0.2em}){4-5}
\cmidrule[0.5pt](lr{0.2em}){6-6}
Method & Obj. & Gap (\%) & Obj. & Gap (\%) & Obj. \\
\midrule
GFACS & \textbf{16.8042} $\pm$ 0.0043 & \textbf{1.56} & \textbf{64.1858} $\pm$ 0.0071 & \textbf{1.95} & \textbf{51.8428} $\pm$ 0.0706 \\
\quad -- \textit{Off-policy Training} & 17.0861 $\pm$ 0.0028 & 3.26 & 64.4248 $\pm$ 0.0155 & 2.33 & 52.3158 $\pm$ 0.1171 \\
\quad -- \textit{Energy Reshaping} & 16.8718 $\pm$ 0.0049 & 1.97 & 64.4121 $\pm$ 0.0069 & 2.31 & 52.0654 $\pm$ 0.0998 \\
\quad -- \textit{Shared $\varepsilon$-norm} & 17.5209 $\pm$ 0.0062 & 5.87 & 65.5069 $\pm$ 0.0171 & 4.05 & 90.7332 $\pm$ 9.1649 \\
\quad\quad + \textit{VarGrad} \cite{zhang2022robust} & 16.8570 $\pm$ 0.0086 & 1.88 & 64.2771 $\pm$ 0.0293 & 2.09 & 51.9504 $\pm$ 0.1222 \\
\bottomrule[1.0pt]
\end{tabular}

}
    \label{tab:ablation_full}
\end{table}

In this section, we present the comprehensive results of our ablation study, including standard deviations from multiple independent runs (three runs) to ensure statistical significance. \cref{tab:ablation_full} demonstrates that each component of our technical proposals contributes to improvement, supported by empirical and statistically rigorous evidence.

\subsection{Results without local search on TSP and CVRP}

\begin{table}[ht]
    \centering
    \caption{Results obtained with and without local search (LS) at test time on TSP and CVRP are reported. The values represent the objective value averaged over 128 test instances and further averaged across three independent runs.}
    \vspace{-5pt}
    \resizebox{0.55\linewidth}{!}{\begin{tabular}{lcccc}
\toprule[1.0pt]
Method & TSP200 & TSP500 & TSP500 & CVRP500 \\
\midrule
ACO w/o LS & 14.19 & 24.05 & 42.95 & 95.24 \\
DeepACO w/o LS & 11.59 & 18.83 & 32.12 & 73.20 \\
GFACS w/o LS & 12.63 & 21.72 & 34.04 & 81.03 \\
\midrule
ACO w/ LS & 10.91 & 17.38 & 28.45 & 64.62 \\
DeepACO w/ LS & 10.77 & 16.84 & 28.44 & 64.49 \\
GFACS w/ LS & 10.75 & 16.80 & 28.32 & 64.19 \\
\bottomrule

\end{tabular}}
    \label{tab:ls_ablation}
    \vspace{0pt}
\end{table}

We conduct an additional ablation study to assess the impact of local search at test time. Note that we use the same policies as in \cref{sec:exp-vs-other-aco} or \cref{sec:exp-vs-nco}, which were trained \textit{with} local search. The results are presented in \cref{tab:ls_ablation}. As expected, DeepACO (RL-based) performs better than ours when local search is not applied. This outcome can be attributed to two factors: (1) The RL-based method tends to generate non-diverse low-energy samples, whereas the GFlowNet-based method (GFACS) produces a more diverse set of candidates; and (2) Due to the energy reshaping (Eq.~(\ref{eq:energy_reshaping}) in \cref{sec:stepA}), the policy is optimized with respect to a modified energy map that assumes the application of local search, resulting in suboptimal performance when local search is omitted. This second factor also explains why GFACS outperforms other methods when local search is applied.

\subsection{Runtime analysis}

\begin{table}[ht]
    \centering
    \caption{Average running time (in seconds) to solve a single instance for~\cref{tab:lib}.}
    \vspace{-5pt}
    \resizebox{0.4\linewidth}{!}{\begin{tabular}{clccc}
\midrule
& \multicolumn{1}{c}{$N$} & ACO & DeepACO & GFACS \\
\midrule
\parbox[t]{2mm}{\multirow{3}{*}{\rotatebox[origin=c]{90}{\scriptsize{TSPLib}}}} & 100-299 & 1.3 & 1.3 & 1.3 \\
& 300-699 & 11.2 & 10.5 & 10.8 \\
& 700-1499 & 89.3 & 72.5 & 77.0 \\
\midrule
\parbox[t]{2mm}{\multirow{3}{*}{\rotatebox[origin=c]{90}{\scriptsize CVRPLib}}} & 100-299 & 15.5 & 15.5 & 15.5 \\
& 300-699 & 34.2 & 35.1 & 35.5 \\
& 700-1001 & 60.0 & 59.8 & 61.8 \\

\bottomrule
\end{tabular}}
    \label{tab:lib_time}
    \vspace{0pt}
\end{table}

For completeness, in \cref{tab:lib_time}, we report the runtime for each algorithm on the real-world instances (\cref{tab:lib} in \cref{sec:exp-vs-other-aco}). Although all algorithms have similar inference times, DeepACO and GFACS occasionally run faster than ACO. This is attributed to the increased complexity of real-world instances: both DeepACO and GFACS benefit from an NN-based amortized prior that enables them to start closer to a local optimum, resulting in earlier termination of the local search. This also explains why DeepACO is slightly faster than GFACS, as it produces a more locally optimized initial solution.

\section{Symmetric solutions of combinatorial optimization}\label{app:symmetric}

This section delves into the symmetric characteristics of combinatorial solutions, wherein multiple trajectories, denoted as \( \tau \), can map to an identical combinatorial solution. We begin with a formal analysis focusing on two quintessential tasks: the Traveling Salesman Problem (TSP) and the Capacitated Vehicle Routing Problem (CVRP). This analysis lays the foundation for understanding the connections to other related combinatorial optimization problems.

\begin{proposition}[Symmetry Solution in TSP (Hamiltonian Cycle)]
Let \( \mathcal{G} = (\mathcal{V}, \mathcal{D}) \) be a finite, undirected graph where \( \mathcal{V} \) is the set of vertices and \( \mathcal{D} \) is the set of edges. Assume \( \mathcal{G} \) contains a Hamiltonian cycle. Denote \( N = |\mathcal{V}| \) as the number of vertices. If \( N > 1 \), there exist exactly \( 2N \) symmetric trajectories (including the cycle and its reverse) starting from each vertex in the cycle. For \( N = 1 \), there is exactly one identical trajectory.
\end{proposition}

\begin{proof}
A Hamiltonian cycle in \( \mathcal{G} \) is a cycle that visits every vertex exactly once. For \( N > 1 \), each cycle can be traversed in two distinct directions (clockwise and counterclockwise) from each vertex, yielding \( 2 \) trajectories per vertex. Since there are \( N \) vertices, this results in \( 2N \) symmetric trajectories. For \( N = 1 \), the cycle is trivial and only has one trajectory.
\end{proof}

\begin{proposition}[Symmetric Solutions in CVRP]
Consider a Capacitated Vehicle Routing Problem (CVRP) defined on a graph \( \mathcal{G} = (\mathcal{V}, \mathcal{D}) \) with a designated depot vertex. Let there be \( K \) distinct Hamiltonian cycles (or routes) originating and terminating at the depot, and denote \( S \) as the number of single vertex cycles (excluding the depot) among these \( K \) cycles. The total number of symmetric solutions, considering the directionality of each cycle, is given by \( K! \times 2^{K-S} \).
\end{proposition}

\begin{proof}
Each of the \( K \) cycles can be ordered in \( K! \) ways. Excluding the \( S \) single vertex cycles, each of the remaining \( K-S \) cycles can be traversed in two directions (forward and reverse). Therefore, for these \( K-S \) cycles, there are \( 2^{K-S} \) ways to choose a direction. The product \( K! \times 2^{K-S} \) gives the total number of symmetric solutions for the CVRP under these conditions.
\end{proof}

The propositions discussed previously can be extended to various variants of vehicle routing problems. Specifically, in the context of the orienteering problem (OP) and prize collecting traveling salesman problem (PCTSP), which represent special cases where \( K = 1 \), the number of symmetric trajectories is determined by the number of vertices, \( N \), in the solution. When \( N > 1 \), there are exactly 2 symmetric trajectories for each route. In contrast, for a single-vertex solution (\( N = 1 \)), the trajectory is identical, reflecting the singularity of the route.

\section{Related works}\label{app:extended}

\subsection{Deep constructive policy for neural combinatorial optimization}

Constructive policy methods systematically build solutions by incrementally adding elements, beginning with an empty set and culminating in a complete solution. This approach ensures compliance with constraints by progressively narrowing the scope of actions, making it versatile for application in domains like vehicle routing \citep{Nazari, kool2018attention}, graph combinatorial optimization \citep{khalil2017learning, ahn2020learning, zhang2023let}, and scheduling \citep{zhang2020learning}.

In deep learning, various techniques adopt this constructive policy framework. Notably, the attention model \citep[AM;][]{kool2018attention} has been pivotal in vehicle routing problems (VRPs), establishing a new standard for stabilizing REINFORCE training. Subsequent advancements in AM have included enhanced decoder architecture and decoding procedure \citep{son2023solving, xin2021multi,luo2023neural,drakulic2023bq}, improved REINFORCE \citep{williams1992simple} baseline assessments \citep{kwon2020pomo, kim2022sym}, leveraging symmetries for sample efficiency \citep{kim2023enhancing}, the application of population-based reinforcement learning (RL) \citep{grinsztajn2022poppy}, ensemble-based learning \citep{jiang2023ensemblebased, zhou2024mvmoe}. Recent advancements in hierarchical algorithms exploit AM-based policies for solution optimization: one policy broadly generates solutions while another refines them locally \citep{kim2021learning,hou2022generalize,ye2024glop, zheng2024udc}.

Graph combinatorial optimization, particularly in problems like maximum independent sets, has diversified from contrive VRPs policies due to their local decomposability and the relative ease of training value functions for partial solutions where VRPs are not locally decomposable due to their global contains. Starting with the influential works by \citet{khalil2017learning}, which employed deep Q networks \citep{mnih2013playing} for graph policy training, creating solutions constructively, \citet{ahn2020learning} introduced an innovative Markov Decision Process (MDP) to generate multiple solution components per iteration by exploiting the local decomposability of graph combinatorial problems. \citet{zhang2023let} further developed this concept, training their MDP with forward-looking GFlowNets \citep{pan2023better}.

Our research aligns with the deep constructive policy framework, where we build a solution from an initial empty solution ($s_0$) based on the policy trained with GFlowNets. Additionally, our method is included in the heatmap-based approaches, detailed in \Cref{app:heatmap}.

\subsection{Deep improvement policy for neural combinatorial optimization}

The improvement policy focuses on refining complete solutions, similar to local search methods like 2opt \citep{Croes58}. Its advantage lies in achieving comparable performances to constructive policies, especially when allotted sufficient time for extensive improvement iterations. Researchers typically adapt existing heuristic local search techniques, creating learning policies that emulate these methods through maximizing expected returns via reinforcement learning \citep{d2020learning,NLNS,ma2021learning,ma2022efficient,ma2023learning,kim2022neuro}. Also, recent works have explored active search to refine solutions via policy adjustments \citep{hottung2021efficient, son2023meta} or sampling the conditions from a continuous latent space \citep{chalumeau2023combinatorial}. Moreover, alternative strategies using discrete Langevin sampling instead of neural networks have shown promising outcomes \citep{pmlr-v202-sun23c}. This approach serves as a complementary extension to deep constructive policies, as it often enhances solutions initially formulated by constructive methods. Additionally, our method, which integrates local search into the training loop, could see further enhancements if advancements in this area are incorporated into our training procedures.

\subsection{Heatmap-based neural combinatorial optimization}\label{app:heatmap}

The heatmap-based method \citep{joshi2021learning} entails the acquisition of a graph representation, referred to as the ``Heatmap," from input graph instances. This Heatmap serves as a prior distribution, influencing the generation of a solution. However, this approach is not self-contained and requires an additional solution construction step, which may involve techniques like beam search \citep{joshi2021learning}, dynamic programming \citep{kool2022deep}, Monte Carlo tree search \citep{fu2021generalize, qiu2022dimes, sun2023difusco, xia2024position} or greedy search with a guide using an objective function \citep{li2023from}. The key advantage of the heatmap-based method lies in its exceptional generalizability when compared to constructive policies. It operates at an abstract level, free from the constraints inherent in generating actual combinatorial solutions.

Our approach falls under heatmap-based methodologies, where we learn a heuristic metric (i.e., heatmap) and employ ant colony optimization for solution construction. A key advantage of our method over many other supervised learning-based heatmap approaches is its independence from optimally labeled data. 

\subsection{GFlowNets for neural combinatorial optimization}

Originally proposed as a framework for learning policies over discrete compositional variables~\citep{bengio2021flow}, GFlowNets have since been applied to combinatorial optimization problems~\citep{zhang2023let,zhang2022robust}. They offer unique advantages over reinforcement learning methods, particularly in generating diverse candidate solutions and enabling off-policy training. Although prior studies have explored the diversity aspect~\citep{zhang2023let,zhang2022robust}, our work is, to our knowledge, the first to empirically validate the benefits of off-policy training in combinatorial optimization settings.

To underscore our contribution, we conceptually compare our method against two recent approaches that incorporate off-policy training for GFlowNets: local search GFlowNets (LS-GFN)~\citep[LS-GFN;][]{kim2023local} and the hindsight-like off-policy strategy proposed by~\citet{zhu2023sample}. Note that, in \cref{sec:exp-fl-db}, we provide an empirical comparison of our approach with FL-DB~\citep{zhang2023let} and VarGrad~\citep{zhang2022robust}.

The key difference between LS-GFN and our approach is that our method addresses conditional inference problems. This requires training conditional GFlowNets that operate under distinct reward landscapes for different conditions, whereas LS-GFN is designed for unconditional problems. Moreover, the local search operators themselves differ. LS-GFN uses a back-and-forth mechanism that partially backtracks trajectories and reconstructs them using its own policy, thereby limiting exploration to sequences sharing identical root segments. In contrast, our local search methods, such as the 2-opt operator for the TSP, explore the solution space within a Hamming ball, allowing for more diverse exploration without the constraint of shared initial sequences. Furthermore, our destroy-and-repair local search method (\cref{append:destroy-and-repair}) can be viewed as an extension of LS-GFN, incorporating symmetricity-aware destruction to mitigate the limitations of the back-and-forth operator.

In \citet{zhu2023sample}, they tries to model a family of policies using hypernetworks to tackle multi-objective optimization problems. They also propose a hindsight-like off-policy strategy, which enables sharing of off-policy samples among policies configured with varying objective mixing coefficients, thereby enhancing sample efficiency during training. In contrast, our method leverages hindsight observations from a (deterministic) local search (\cref{sec:stepA}). Specifically, we adjust the energy of an original sample by incorporating the energy after local search, thereby quantifying how readily the sample can be converted into a low-energy solution.

\end{document}